\documentclass[10pt]{article} 
\usepackage[accepted]{tmlr}

\usepackage{hyperref}
\usepackage{url}

\usepackage{amsmath,amsfonts,bm}

\def\1{\bm{1}}

\DeclareMathAlphabet{\mathsfit}{\encodingdefault}{\sfdefault}{m}{sl}
\SetMathAlphabet{\mathsfit}{bold}{\encodingdefault}{\sfdefault}{bx}{n}

\def\gD{{\mathcal{D}}}

\DeclareMathOperator*{\argmax}{arg\,max}
\DeclareMathOperator*{\argmin}{arg\,min}

\usepackage{hyperref}
\usepackage{url}
\usepackage{wrapfig}
\usepackage{amsmath,amssymb,amsthm,amsfonts}
\usepackage{dsfont}

\theoremstyle{plain}
\newtheorem{theorem}{Theorem}[section]

\newtheorem{lemma}[theorem]{Lemma}
\newtheorem{corollary}[theorem]{Corollary}
\newtheorem{definition}[theorem]{Definition}
\newtheorem{assumption}[theorem]{Assumption}

\newtheorem{remark}[theorem]{Remark}

\newcommand{\Ac}{\mathcal{A}}

\newcommand{\Dc}{\mathcal{D}}
\newcommand{\Ec}{\mathcal{E}}
\newcommand{\Fc}{\mathcal{F}}

\newcommand{\Lc}{\mathcal{L}}

\newcommand{\Rc}{\mathcal{R}}
\newcommand{\Sc}{\mathcal{S}}

\newcommand{\Xc}{\mathcal{X}}

\newcommand{\Pf}{\mathfrak{P}}

\newcommand{\Rf}{\mathfrak{R}}

\newcommand{\oneb}{\mathds{1}}

\newcommand{\Eb}{\mathbb{E}}

\newcommand{\Nb}{\mathbb{N}}

\newcommand{\Rb}{\mathbb{R}}

\newcommand{\Vb}{\mathbb{V}}

\usepackage{algorithm}
\usepackage{algpseudocode}

\newcommand{\Reg}{\textup{Reg}}

\usepackage{colortbl}
\usepackage{graphicx}
\usepackage{caption}
\usepackage{subfigure}
\usepackage{enumitem}
\usepackage{tikz}
\usetikzlibrary{fit,calc}
\newcommand*{\tikzmk}[1]{\tikz[remember picture,overlay,] \node (#1) {};\ignorespaces}
\newcommand{\boxita}[1]{\tikz[remember picture,overlay]{\node[yshift=0pt,fill=#1,opacity=.08,fit={($(A)+(-0.01\linewidth, 0.5\baselineskip)$)($(A)+(0.97\linewidth, -8\baselineskip)$)}] {};}\ignorespaces}
\newcommand{\boxitb}[1]{\tikz[remember picture,overlay]{\node[yshift=0pt,fill=#1,opacity=.15,fit={($(B)+(0.46\linewidth, -0.1\baselineskip)$)($(B)+(-0.52\linewidth, 2.5\baselineskip)$)}] {};}\ignorespaces}
\usepackage{makecell}
\usepackage{arydshln}
\allowdisplaybreaks

\title{Harnessing the Power of Federated Learning in Federated Contextual Bandits}

\author{\name Chengshuai Shi \email cs7ync@virginia.edu \\
      \addr Department of Electrical and Computer Engineering\\
      University of Virginia
      \AND
      \name Ruida Zhou \email ruida@g.ucla.edu \\
      \addr Department of Electrical and Computer Engineering\\
      University of California, Los Angeles
      \AND
      \name Kun Yang \email ky9tc@virginia.edu \\
      \addr Department of Electrical and Computer Engineering\\
      University of Virginia
      \AND\name Cong Shen \email cong@virginia.edu \\
      \addr Department of Electrical and Computer Engineering\\
      University of Virginia}

\def\month{07}  
\def\year{2024} 
\def\openreview{\url{https://openreview.net/forum?id=Z8wcREe9qV}} 

\begin{document}

\maketitle

\begin{abstract}
    Federated learning (FL) has demonstrated great potential in revolutionizing distributed machine learning, and tremendous efforts have been made to extend it beyond the original focus on supervised learning. Among many directions, federated contextual bandits (FCB), a pivotal integration of FL and sequential decision-making, has garnered significant attention in recent years. Despite substantial progress, existing FCB approaches have largely employed their tailored FL components, often deviating from the canonical FL framework. Consequently, even renowned algorithms like FedAvg remain under-utilized in FCB, let alone other FL advancements. Motivated by this disconnection, this work takes one step towards building a tighter relationship between the canonical FL study and the investigations on FCB. In particular, a novel FCB design, termed FedIGW, is proposed to leverage a regression-based CB algorithm, i.e., inverse gap weighting. Compared with existing FCB approaches, the proposed FedIGW design can better harness the entire spectrum of FL innovations, which is concretely reflected as (1) flexible incorporation of (both existing and forthcoming) FL protocols; (2) modularized plug-in of FL analyses in performance guarantees; (3) seamless integration of FL appendages (such as personalization, robustness, and privacy). We substantiate these claims through rigorous theoretical analyses and empirical evaluations.
\end{abstract}

\section{Introduction}\label{sec:intro}
Federated learning (FL), initially proposed by \citet{mcmahan2017communication, konevcny2016federated}, has garnered significant attention for its effectiveness in enabling distributed machine learning with heterogeneous agents \citep{li2020federated,kairouz2021advances}. As FL has gained popularity, numerous endeavors have sought to extend its applicability beyond the original realm of supervised learning, e.g., to unsupervised and semi-supervised learning \citep{zhang2020federated, van2020towards,zhuang2022divergence,lubana2022orchestra}. 
Among these directions, the exploration of federated contextual bandits (FCB) has emerged as a particularly compelling area of research, representing a pivotal fusion of FL and sequential decision-making, which has found various practical applications in cognitive radio and recommendation systems, among others.

Over the past several years, substantial progress has been made in the field of FCB \citep{wang2019distributed, li2022generalized, li2022kernelized,li2023learning, dai2023federated}, particularly those involving varying function approximations (e.g., linear models, as discussed in \citet{huang2021federated, dubey2020differentially,li2022asynchronous, he2022simple, amani2022distributed,fan2023federated}). Despite their different focuses, it can be observed that these existing designs all employ certain FL components to enable the participating agents to collaboratively update their CB parameterization via locally collected interaction data. 

However, these FL components adopted in the previous FCB works are often over-simplified. In particular, the canonical FL framework (traced back to the celebrated FedAvg algorithm \citep{mcmahan2017communication}) typically takes an optimization view of incorporating the local data through \emph{multi-round} aggregation of \emph{model parameters} (such as gradients). In contrast, the FL protocol in many existing FCB works is \emph{one-shot} aggregation of some \emph{compressed local data} per epoch (e.g., combining local estimates and local covariance matrices in the study of federated linear bandits). Admittedly, for some simple cases, such straightforward aggregation is sufficient and allows problem-specific finetuning for tight performance bounds. However, such a deviation from the canonical FL studies prohibits existing FCB designs from leveraging the vast FL advances, and thus largely limits the connection between FL and FCB.

Motivated by this disconnection, this work, instead of pursuing tighter performance bounds, aims to utilize the canonical FL framework as the FL component of FCB to harness the full power of FL studies in FCB. We propose FedIGW -- an exploring design that demonstrates the ability to leverage a comprehensive array of FL advancements, encompassing canonical algorithmic approaches (like FedAvg \citep{mcmahan2017communication} and SCAFFOLD \citep{karimireddy2020scaffold}), rigorous convergence analyses, and critical appendages (such as personalization, robustness, and privacy). To the best of our knowledge, this is the first paper that explicitly focuses on the close connection between FL and FCB, which we hope can inspire a new line of FCB studies. The distinctive contributions of FedIGW can be succinctly summarized as follows:

$\bullet$ \textbf{Flexible incorporation of FL protocols.} In the FCB setting with stochastic contexts and a realizable reward function, FedIGW employs the inverse gap weighting (IGW) algorithm for CB while versatile FL protocols can be incorporated (e.g., FedAvg and SCAFFOLD), provided they can solve a standard FL problem. These two parts iterate according to designed epochs: FL, drawing from previously gathered interaction data, supplies estimated reward functions for the forthcoming IGW interactions. A pivotal advantage is that the flexible FL component in FedIGW provides substantial adaptability, meaning that existing and future FL protocols can be seamlessly leveraged. Experimental results using real-world data with several different FL choices corroborate the practicability and flexibility of FedIGW.

$\bullet$ \textbf{Modularized plug-in of FL analyses.} A general theoretical analysis of FedIGW is developed to demonstrate its provably efficient performance. The influence of the adopted FL protocol is captured through its optimization error, delineating the excess risk of the learned reward function. Notably, any theoretical breakthroughs in FL convergence rates can be immediately integrated into the obtained analysis framework and supply the corresponding guarantees of FedIGW. Concretized results are further provided through the utilization of FedAvg and SCAFFOLD in FedIGW.

$\bullet$ \textbf{Seamless integration of FL appendages.} Beyond its inherent generality and efficiency, FedIGW exhibits exceptional extensibility. Various appendages from FL studies can be flexibly integrated without necessitating alterations to the CB component. We explore the extension of FedIGW to personalized learning and the incorporation of privacy and robustness guarantees. Similar investigations in prior FCB works would entail substantial algorithmic modifications, while FedIGW can effortlessly leverage corresponding FL advancements to obtain these appealing attributes.

\textbf{Key related works.} Most of the previous studies on FCB are discussed in Sec.~\ref{sec:principle}, and more comprehensively reviewed in Appendix~\ref{app:related}. We note that these FCB designs with tailored FL protocols in previous works sometimes can achieve near-optimal performance bounds in specific settings, while our proposed FedIGW is more practical and extendable. We believe these two types of designs are valuable supplements to each other. A high-level comparison between the proposed FedIGW and existing FCB designs is listed in Table~\ref{tab:comparison}. 

{Here we particularly note a recent paper \citep{agarwal2023empirical} that is closely related to this work. It also proposes to decouple the FL components in FCB by leveraging regression-based CB designs. However, \citet{agarwal2023empirical} mainly focuses on empirical investigations, while our work offers valuable complementary contributions by conducting thorough theoretical analyses (see Sec.~\ref{sec:theory}), building a modularized connection between theoretical studies in FL and FCB. Moreover, experiments reported in Sec.~\ref{sec:exp} provide empirical results on two datasets that are different than \citet{agarwal2023empirical}, offering additional practical insights.}

\section{Federated Contextual Bandits}\label{sec:FCB}
This section introduces federated contextual bandits (FCB). A concise formulation is first provided. Then, the existing works are re-visited with a focus on revealing the disconnection between FL and FCB.

\subsection{Problem Formulation}\label{subsec:formulation}

\textbf{Agents.} In the FCB setting, a total of $M$ agents simultaneously participate in solving a contextual bandit (CB) problem. For generality, we consider an asynchronous system: each of the $M$ agents has a clock indicating her time step, which is denoted as $t_m = 1, 2, \cdots$ for agent $m$. For convenience, we also introduce a global time step $t$. Denote by $t_m(t)$ the agent $m$'s local time step when the global time is $t$, and $t(t_m, m)$ the global time step when the agent $m$'s local time is $t_m$.

Agent $m$ at each of her local time step $t_m = 1,2, \cdots$ observes a context $x_{m,t_m}$, selects an action $a_{m,t_m}$ from an action set $\Ac_{m,t_m}$, and then receives the associated reward $r_{m,t_m}(a_{m,t_m})$ (possibly depends on both $x_{m,t_m}$ and $a_{m, t_m}$) as in the standard CB \citep{lattimore2020bandit}. Each agent's goal is to collect as many rewards as possible given a time horizon.

\textbf{Federation.} While many efficient single-agent (centralized) algorithms have been proposed for CB \citep{lattimore2020bandit}, FCB targets building a federation among agents to perform collaborative learning such that the performance can be improved from learning independently. Especially, common interests shared among agents motivate their collaboration. Thus, FCB studies typically assume that the agents' environments are either fully \citep{wang2019distributed, huang2021federated,dubey2020differentially,he2022simple, amani2022distributed, li2022kernelized, li2022generalized, dai2023federated} or partially \citep{li2022asynchronous,agarwal2020federated} shared in the global federation.

In federated learning, the following two modes are commonly considered: (1) There exists a central server in the system, and the agents can share information with the server, which can then broadcast aggregated information back to the agents; or (2) There exists a communication graph between agents, who can share information with their neighbors on the graph. In the later discussions, we mainly consider the first scenario, i.e., collaborating through the server, which is also the main focus in FL, while both modes can be effectively encompassed in the proposed FedIGW design.

\subsection{The Current Disconnection Between FCB and FL}\label{sec:principle}
The exploration of FCB traces its origins to distributed multi-armed bandits \citep{wang2019distributed}. Since then, FCB research has predominantly focused on enhancing performance in broader problem domains, encompassing various types of reward functions, such as linear \citep{wang2019distributed, huang2021federated,dubey2020differentially}, kernelized \citep{li2022kernelized,li2023learning}, generalized linear \citep{li2022generalized} and neural \citep{dai2023federated} (see Appendix~\ref{app:related} for a comprehensive review).

\begin{table}[tb]
    \centering
    \caption{\centering A comparison between existing FCB designs and the proposed FedIGW.}
    \begin{tabular}{c|c|c}
        \hline
         & Existing FCB designs & FedIGW \\
        \hline
        FL components & Develop tailored FL protocols & \makecell[c]{Leverage versatile FL protocols, \\such as FedAvg and SCAFFOLD} \\
        \hline
        Theoretical guarantees &\makecell[c]{Analyse tailored FL protocols \\for the focused instance} & \makecell[c]{Plugin FL convergence rates\\ in a modularized fashion}\\
        \hline
        \makecell[c]{Extensions (e.g., personalization, \\robustness, privacy)} & Require further tailored protocols & \makecell[c]{Integrate corresponding \\FL advances directly} \\
        \hline
    \end{tabular}
     \vspace{-0.2in}
    \label{tab:comparison}
\end{table}

\begin{table}[tb]
    \centering
    \caption{\centering A compact summary of investigations on FCB with their adopted FL and CB components; \newline a more comprehensive review is in Appendix~\ref{app:related}.}
    \begin{tabular}{c|c|>{\columncolor[RGB]{251, 229, 214}}c|>{\columncolor[RGB]{222, 235, 247}}c}
        \hline
        Reference & Setting & FL & CB \\
        \hline
        \multicolumn{4}{c}{Globally Shared Full Model (See Section~\ref{sec:FedIGW})}\\
        \hline
       \cite{wang2019distributed}  & Tabular & Mean Averaging & AE\\
       \cite{wang2019distributed, huang2021federated} & Linear & Linear Regression & AE\\
       \makecell{\cite{li2022asynchronous, he2022simple}} & Linear &  Ridge Regression & UCB\\
       \cite{li2022generalized} & Gen. Linear & Distributed AGD & UCB\\
       \cite{li2022kernelized,li2023learning} & Kernel & Nystr\"om Approximation & UCB\\
       \cite{dai2023federated} & Neural & NTK Approximation & UCB\\
       FedIGW (this work) & Realizable & Flexible (e.g., FedAvg) & IGW \\
       \hline
        \multicolumn{4}{c}{Globally Shared Partial Model (see Section~\ref{subsec:per})}\\
        \hline
        \cite{li2022asynchronous} & Linear &  Alternating Minimization & UCB\\
       \cite{agarwal2020federated} & Realizable  & FedRes.SGD & $\varepsilon$-greedy\\
       FedIGW (this work) & Realizable & Flexible (e.g., LSGD-PFL) & IGW\\
       \hline
    \end{tabular}
    \begin{center}
        { AE: arm elimination; Gen. Linear: generalized linear model; AGD: accelerated gradient descent}
    \end{center}
    \label{tab:summary}
    \vspace{-0.2in}
\end{table}

\begin{wrapfigure}{R}{0.45\textwidth}	
        \setlength{\abovecaptionskip}{-2pt} 
	\centering
	\includegraphics[width=0.45\textwidth]{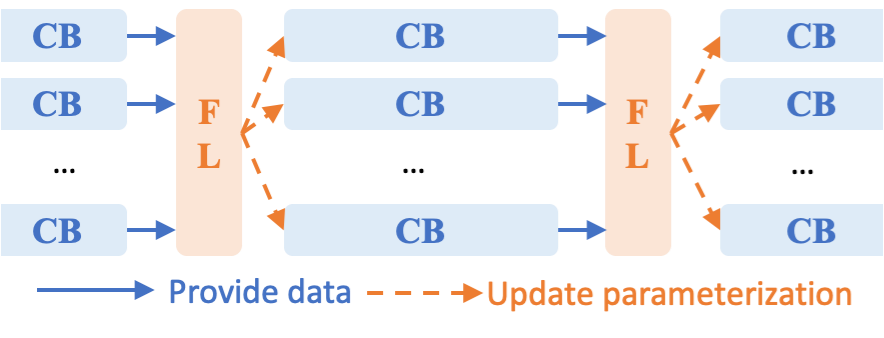}
	\caption{\centering The FCB design principle of periodically alternating between the employed CB and FL components.}
        \label{fig:principle}
\end{wrapfigure}

Upon a holistic review of these works, it becomes apparent that each of them employs a particular FL protocol to update the parameters required by CB.  To be more specific, a periodically alternating design between CB and FL is commonly adopted as reflected in Fig.~\ref{fig:principle}: \textcolor{cyan}{CB} (collects one epoch of data in parallel) $\rightarrow$ \textcolor{orange}{FL} (proceeds with CB data together and outputs CB's parameterization) $\rightarrow$ updated \textcolor{cyan}{CB} (collects another epoch of data in parallel) $\rightarrow$ $\cdots$. A compact summary, including the components of FL and CB employed in previous FCB works, is presented in Table~\ref{tab:summary}.

However, with a deeper look into the existing works, it is evident that the adopted FL components are not well investigated and even have some mismatches from canonical FL designs \citep{mcmahan2017communication,konevcny2016federated}. For example, in federated linear bandits \citep{wang2019distributed,dubey2020differentially,li2022asynchronous,he2022simple, amani2022distributed,fan2023federated} and its extensions \citep{li2022kernelized,li2023learning,li2022generalized,dai2023federated}, the adopted FL protocols typically involve the direct transmission of local reward aggregates and covariance matrices, constituting a \emph{one-shot aggregation} of \emph{compressed local data} per epoch (albeit with subtle variations, such as synchronous or asynchronous communications); a concrete example is given in Appendix~\ref{subapp:example},  Due to both efficiency and privacy concerns, such choices are rare (and even undesirable) in canonical FL studies, where agents typically communicate and aggregate their \emph{model parameters} (e.g., gradients) over \emph{multiple rounds}, e.g., the renowned FedAvg algorithm \citep{mcmahan2017communication} (see details in Appendix~\ref{subapp:example}).

We believe that this disparity represents a significant drawback in current FCB studies, as it limits the connection between FL and FCB to merely philosophical, i.e., benefiting individual learning by collaborating through a federation, while vast FL studies cannot be leveraged to benefit FCB as illustrated in Fig.~\ref{fig:comparison}. Driven by this gap, this work aims to take one step towards establishing a closer relationship between FCB and FL through the introduction of an exploring design, FedIGW, that is detailed in the subsequent sections. This approach provides the flexibility to integrate any FL protocol following the standard FL framework, which allows us to effectively harness the progress made in FL studies, encompassing canonical algorithmic designs, convergence analyses, and useful appendages.

\begin{figure}[bth]
    \centering
    \includegraphics[width=0.95\linewidth]{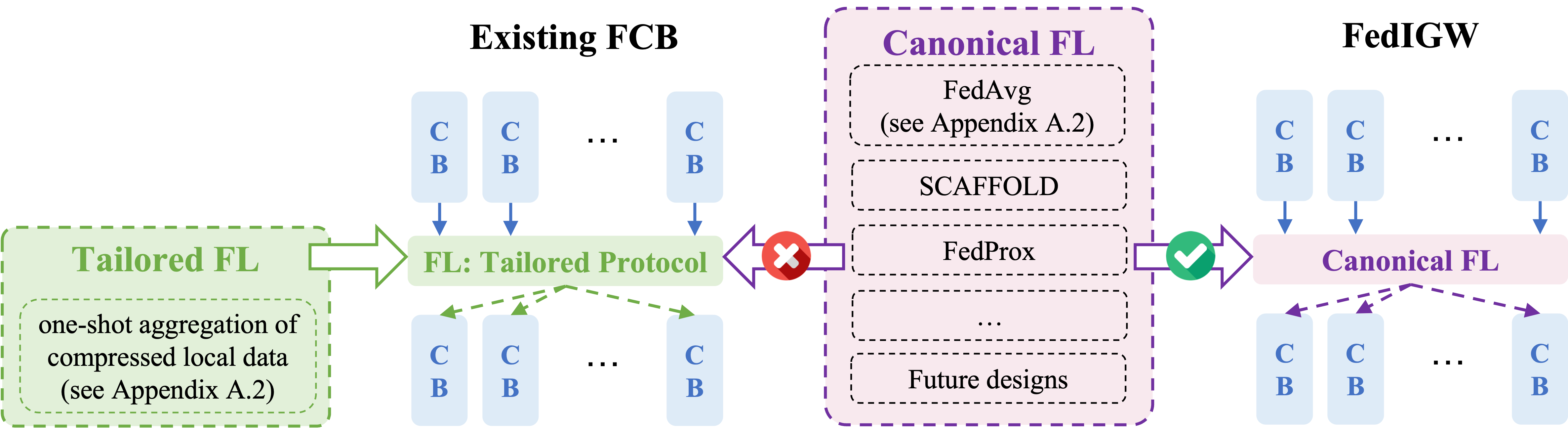}
    \caption{\centering Comparison between the FL components in existing FCB approaches and the FedIGW design proposed in this work, where the former requires tailored FL protocols while the latter can flexibly leverage both existing and forthcoming protocols in canonical FL studies. Additional comparisons regarding the FL components can be found in Appendix~\ref{subapp:example}.}
    \label{fig:comparison}
    \vspace{-0.2in}
\end{figure}

\section{FedIGW: Flexible Incorporation of FL Protocols}\label{sec:FedIGW}
In this section, we present FedIGW, a novel FCB algorithm proposed in this work. Before delving into the algorithmic details, a more concrete system model with stochastic contexts and a realizable reward function is introduced. Subsequently, we outline the specifics of FedIGW, emphasizing its principal strength in seamlessly integrating canonical FL protocols.

\subsection{System Model}\label{subsec:system}
Built on the formulation in Sec.~\ref{sec:FCB}, for each agent $m\in [M]$, denote $\Xc_m$ a context space, and $\Ac_m$ a finite set of $K_m$ actions. At each time step $t_m$ of each agent $m$, the environment samples a context $x_{m,t_m} \in \Xc_m$ and a context-dependent reward vector $r_{m,t_m}\in [0,1]^{\Ac_m}$ according to a fixed but unknown distribution $\Dc_m$. The agent $m$, as in Sec.~\ref{sec:FCB}, then observes the context $x_{m,t_m}$, picks an action $a_{m,t_m}\in \Ac_m$, and receives the reward $r_{m,t_m}(a_{m,t_m})$. The expected reward of playing action $a_m$ given context $x_m$ is denoted as $\mu_m(x_m,a_m): = \Eb[r_{m,t_m}(a_m)|x_{m,t_m}=x_m]$.

With no prior information about the rewards, the agents gradually learn their optimal policies, denoted by $\pi^*_m(x_m):= \argmax_{a_m\in \Ac_m} \mu_m(x_m, a_m) $ for agent $m$ with context $x_{m}$. Following a standard notation \citep{wang2019distributed, huang2021federated, dubey2020differentially, li2022asynchronous, he2022simple, amani2022distributed, li2022generalized, li2022kernelized,li2023learning, dai2023federated}, the overall regret of $M$ agents in this environment is
\begin{align*}
    \Reg(T): = \Eb\left[\sum_{m\in [M]}\sum_{t_m\in [T_m]} \big[\mu_m(x_{m,t_m}, \pi^*_m(x_{m,t_m})) - \mu_m(x_{m,t_m},a_{m,t_m})\big]\right],
\end{align*}
where $T_m = t_m(T)$ is the effective time horizon for agent $m$ given a global horizon $T$ and the expectation is taken over the randomness in contexts and rewards and the agents' algorithms. 
This overall regret can be interpreted as the sum of each agent $m$'s individual regret with respect to (w.r.t.) her optimal strategy $\pi^*_m$. Hence, it is ideal to be sub-linear w.r.t. the number of agents $M$, which indicates the agents' learning processes are accelerated on average due to federation.

\textbf{Realizablilty.}
Despite not knowing the true expected reward functions, we consider the scenario that they are the same across agents and are within a function class $\Fc$, to which the agents have access. This assumption, rigorously stated in the following, is often referred to as the \emph{realizability} assumption.
\begin{assumption}[Realizability]\label{asp:realizable}
    There exists $f^*$ in $\Fc$ such that $f^*(x_m, a_m) = \mu_m(x_m, a_m)$ for all $m\in [M]$, $x_m \in \Xc_m$ and $a_m \in \Ac_m$.
\end{assumption}
This assumption is a natural extension from its commonly-adopted single-agent version \citep{agarwal2012contextual,simchi2022bypassing,xu2020upper,sen2021top} to a federated one. Note that it does not imply that the agents' environments are the same since they may face different contexts $\Xc_m$, arms $\Ac_m$, and distributions $\Dc_m^{\Xc_m}$, where $\Dc^{\Xc_m}_m$ is the marginal distribution of the joint distribution $\Dc_m$ on the context space $\Xc_m$. We study a general FCB setting only with this assumption, which incorporates many previously studied FCB scenarios as special cases. 
For example, the federated linear bandits \citep{huang2021federated, dubey2020differentially,li2022asynchronous, he2022simple, amani2022distributed} are with a linear function class $\Fc$. {Furthermore, Assumption~\ref{asp:realizable} aligns with considerations in the canonical FL studies, where all clients learn one common model (i.e., $f^*$) although their data distributions can be different (i.e., varying distribution $\gD_m$). Additional studies on personalization can be found in Sec.~\ref{subsec:per}.}

\begin{algorithm}[tb]
	\caption{FedIGW (Agent $m$)}
	\label{alg:FedIGW_agent}
	\begin{algorithmic}[1]
        \Require epoch number $l = 1$, reward function $\widehat{f}_{m}^l(\cdot, \cdot)=0$, local dataset $\Sc^{l}_m = \emptyset$
        \For{time step $t_m = 1, 2, \cdots$}
        \State\tikzmk{A} observe context $x_{m,t_m}$ \Comment{\textit{\textcolor{cyan}{CB: IGW}}}
        \State compute $\widehat{a}^*_m = \argmax_{a_m\in \Ac_m} \widehat{f}^{l}(a_m, x_{m,t_m})$ and action selection distribution $$p^{l}_m(a_{m}|x_{m,t_m}) \gets \begin{cases} 1/\left(K_m + \gamma^l\left(\widehat{f}^l(\widehat{a}^*_{m}, x_{m,t_m}) -  \widehat{f}^l(a_m, x_{m,t_m})\right)\right) & \text{if $a_m \neq \widehat{a}^*_m$} \\ 1 - \sum_{a'_m\neq \widehat{a}^*_m} p^{l}_m(a'_{m}|x_{m,t_m})& \text{if $a_m =\widehat{a}^*_m$}\end{cases}$$
        \State select action $a_{m,t_m} \sim p_{m}^l(\cdot|x_{m,t_m})$; observe reward $r_{m,t_m}(a_{m,t_m})$
        \State update the local dataset $\Sc^l_m \gets \Sc^l_m \cup \{(x_{m,t_m}, a_{m,t_m}, r_{m,t_m}(a_{m,t_m}))\}$\tikzmk{B}
        \boxita{cyan}
        \If{$t_m = t_m(\tau^l)$} \Comment{\textit{\textcolor{orange}{FL}}}
        \State perform FL $\widehat{f}^{l+1} \gets \texttt{FLroutine}(\Sc_m^l)$
        \State\tikzmk{A} update dataset $\Sc^{l+1}_m \gets \emptyset$; update epoch $l \gets l+1$\tikzmk{B}
        \EndIf \boxitb{orange} 
        \EndFor
	\end{algorithmic}
	\end{algorithm}
    \subsection{Algorithm Design}\label{subsec:alg}
    The FedIGW algorithm proceeds in epochs, which are separated at time slots $\tau^1, \tau^2, \cdots$ w.r.t. the global time step $t$, i.e., the $l$-th epoch starts from $t = \tau^{l-1}+1$ and ends at $t = \tau^{l}$. The overall number of epochs is denoted as $l(T)$. In each epoch $l$, we describe the FL and CB components as follows, while emphasizing that the FL component is decoupled and follows the standard FL framework.
    
    \textbf{CB: inverse gap weighting (IGW).} For CB, we use inverse gap weighting \citep{abe1999associative}, which has received growing interest in the single-agent setting recently \citep{foster2020beyond,simchi2022bypassing, krishnamurthy2021adapting,ghosh2021model} but has not been fully investigated in the federated setting. At any time step in epoch $l$, when encountering the context $x_m$, agent $m$ first identifies the optimal arm by $\widehat{a}^*_m = \argmax_{a_m \in \Ac_m} \widehat{f}^{l}(x_m, a_m)$ from an estimated reward function $\widehat{f}^{l}$ (provided by the to-be-discussed FL component). Then, she randomly selects her action $a_m$ according to the following distribution, which is inversely proportional to each action's  estimated reward gap from the identified optimal action $\widehat{a}^*_m$:
    \begin{equation*}
        \begin{aligned}
        p^{l}_m(a_{m}|x_{m}) \gets \begin{cases} 1/\left(K_m + \gamma^l\left(\widehat{f}^l(\widehat{a}^*_{m}, x_{m}) -  \widehat{f}^l(a_m, x_{m})\right)\right) & \text{if $a_m \neq \widehat{a}^*_m$} \\ 1 - \sum_{a'_m\neq \widehat{a}^*_m} p^{l}_m(a'_{m}|x_{m})& \text{if $a_m =\widehat{a}^*_m$}\end{cases},
    \end{aligned}
    \end{equation*}
    where $\gamma^l$ is the learning rate in epoch $l$ that controls the exploration-exploitation tradeoff. 
    
    Besides being a valuable supplement to the currently dominating UCB-based studies in FCB, the main merit of leveraging IGW as the CB component is that it only requires an estimated reward function instead of other complicated data analytics, e.g., upper confidence bounds.
    
    \textbf{FL: flexible choices.} By IGW, each agent $m$ performs local stochastic arm sampling and collects a set of data samples $\Sc^{l}_m:=\{(x_{m,t_m}, a_{m,t_m}, r_{m,t_m}: t_m\in [t_m(\tau^{l-1})+1, t_m(\tau^l)])\}$ in epoch $l$. To enhance the performance of IGW in the subsequent epoch $l+1$, an improved estimate $\widehat{f}^{l+1}$ based on all agents' data is desired. This objective aligns precisely with the aim of canonical FL studies, which aggregates local data for better global estimates \citep{mcmahan2017communication,konevcny2016federated}. Thus, the agents can target solving the following standard FL problem:
    \begin{align}\label{eqn:FL}
        \min_{f\in \Fc} \widehat{\Lc}(f; \Sc^l_{[M]}) := \sum_{m\in [M]}(n_m/n)\cdot \widehat{\Lc}_m(f; \Sc^l_m),
    \end{align}
    where $n_m := |\Sc^l_m|$ is the number of samples in dataset $\Sc^l_m$, $n := \sum_{m\in [M]}n_m$ is the total number of samples, and $\widehat{\Lc}_m(f; \Sc^l_m) := (1/n_m) \cdot \sum_{i\in [n_m]} \ell_m(f(x_m^i, a_m^i); r_m^i)$
    is the empirical local loss of agent $m$ with $\ell_m(\cdot; \cdot): \Rb^2 \to \Rb$ as the loss function and $(x_m^i, a_m^i, r_m^i)$ as the $i$-th sample in $\Sc^l_m$. 
    
    As Eqn.~\eqref{eqn:FL} exactly follows the standard formulation of FL, the agents and the server can employ any FL protocol to solve this optimization, such as FedAvg \citep{mcmahan2017communication}, SCAFFOLD \citep{karimireddy2020scaffold} and FedProx \citep{li2020federated}. These wildly-adopted FL protocols typically perform iterative communications of local model parameters (e.g., gradients), instead of one-shot aggregations of compressed local data in previous FCB studies.
    To highlight the remarkable flexibility, we denote the adopted FL protocol as $\texttt{FLroutine}(\cdot)$. With datasets $\Sc^l_{[M]}:=\{\Sc^l_m: m\in [M]\}$, the output function of this FL process, denoted as $\widehat{f}^{l+1}\gets \texttt{FLroutine}(\Sc^l_{[M]})$, is used as the estimated reward function for IGW sampling in the next epoch $l+1$.

    The FedIGW algorithm for agent $m$ is summarized in Alg.~\ref{alg:FedIGW_agent}. The key, as aforementioned, is that the component of FL in FedIGW is highly flexible as it only requires an estimated reward function for later IGW interactions. In particular, any existing or forthcoming FL protocol following the standard FL framework in Eqn.~\eqref{eqn:FL} can be leveraged as the $\texttt{FLroutine}(\cdot)$ in FedIGW.

    \begin{remark}\label{rmk:IGW}\normalfont
        The main underlying reason for selecting IGW  as the CB component is that it is a regression-based CB algorithm, i.e., IGW only requires a learned reward function $\hat{f}^l$ for the CB interaction in one epoch $l$. The canonical FL framework with an optimization perspective is exactly targeted at learning such a function via collaboratively solving Eqn.~\eqref{eqn:FL}, which thus can be integrated with IGW. In contrast, previous FCB designs are predominated by UCB-based CB components as reflected in Table~\ref{tab:summary}. However, obtaining the upper confidence bounds (UCBs) estimates for an unknown reward function is not usually the target of the canonical FL framework. Thus, tailored FL components are developed to fulfill this purpose, e.g., sharing covariance matrices to obtain UCBs for linear reward functions. {We note that there are also other regression-based CB algorithms, e.g., greedy and softmax. IGW is adopted here mainly due to its theoretical superiority demonstrated in Sec.~\ref{sec:theory}, while its strong empirical performances have also been observed in Sec.~\ref{sec:exp}.}
    \end{remark}

    \section{Theoretical Guarantees: Modularized Plug-in of FL Analyses}\label{sec:theory}
    In this section, we theoretically analyze the performance of the FedIGW algorithm, where the impact of the adopted FL choice is modularized as a plug-in component of its optimization error. 

    \subsection{A General Guarantee}\label{subsec:general}
    Denoting $E^l_m := t_m(\tau^{l}) - t_m(\tau^{l-1})$ as the length of epoch $l$ for agent $m$, $E^l_{[M]} : = \{E^l_m: m\in [M]\}$ as the epoch length set, $\underline{c} := \min_{m\in [M], l\in [2, l(T)]} E^l_m/E^{l-1}_m$, $\overline{c} := \max_{m\in [M], l\in [2, l(T)]} E^l_m/E^{l-1}_m$ and $c := \overline{c}/\underline{c}$, the following global regret guarantee can be established.
    \begin{theorem}\label{thm:global_regret}
    Using a learning rate $\gamma^ l = O\left(\sqrt{\sum_{m\in [M]}E^{l-1}_m K_m/(\sum_{m\in [M]}E^{l-1}_m\Ec(E^{l-1}_{[M]}))} \right)$ in epoch $l$, denoting $\bar{K}^l := \sum_{m\in [M]}E^l_m K_m/\sum_{m\in [M]}E^l_m$,
    the regret of FedIGW can be bounded as
        \begin{align}\label{eqn:general_regret}
            \Reg(T) = O\left(\sum_{m\in [M]}E^1_m  + \sum_{l\in [2,l(T)]}c^\frac{5}{2}\sqrt{\bar{K}^l \Ec(E^{l-1}_{[M]})}\sum_{m\in [M]}E^l_m\right).
        \end{align}
        Here $\Ec(E^{l}_{[M]})$ (abbreviated from $\Ec(\Fc; E^{l}_{[M]})$) denotes the excess risk of the output from the adopted $\textup{\texttt{FLroutine}}(\Sc^l_{[M]})$ using the datasets $\Sc^{l}_{[M]}$, whose formal definition is deferred to Definition~\ref{def:error_general}.
    \end{theorem}
    It can be observed that in Eqn.~\eqref{eqn:general_regret}, the first term bounds the regret in the first epoch. The obtained bounds for the regrets incurred within each later epoch (i.e., the term inside the sum over $l$ in the second epoch) can be interpreted as the epoch length times the expected per-step suboptimality, which then relates to the estimation quality of $\widehat{f}^l$ and thus $\Ec(E^{l-1}_{[M]})$ as $\widehat{f}^l$ is learned with the interaction data collected from epoch $l-1$ as in the design of FedIGW shown in Alg.~\ref{alg:FedIGW_agent}.

    \subsection{Some Concretized Discussions}\label{subsec:concrete}
    Theorem~\ref{thm:global_regret} is notably general in the sense that a corresponding regret can be established as long as an upper bound on the excess risk $\Ec(E^{l-1}_{[M]})$ can be obtained for a certain class of reward functions and the adopted FL protocol. In the following, we provide several more concrete illustrations, and especially, a modularized framework to leverage FL convergence analyses.
    To ease the notation, we discuss synchronous systems with a shared number of arms in the following, i.e., $t_m =t, \forall m\in [M]$, and $K_m = K, \forall m\in [M]$,  while noting similar results can be easily obtained for general systems. With this simplification, we can unify all $E^l_m$ as $E^l$ and $\bar{K}^l$ as $K$. 
    
    To initiate the concretized discussions, we start with considering a finite function class $\Fc$, i.e., $|\Fc| <\infty$, which can be extended to a function class $\Fc$ with a finite covering number of the metric space $(\Fc, l_\infty)$. In particular, the following corollary can be established via establishing $\Ec(n_{[M]})= O(\log(|\Fc|n)/n)$ in the considered case as in Lemma~\ref{lem:finite_error}.
    \begin{corollary}[A Finite Function Class]\label{col:finite}
        If $|\Fc|< \infty$ and the adopted FL protocol provides an exact minimizer for Eqn.~\eqref{eqn:FL} with quadratic losses, with $\tau^l = 2^l$, FedIGW incurs a regret of $\Reg(T) = O(\sqrt{KMT\log(|\Fc|MT)})$ and a total $O(\log(T))$ calls of the adopted FL protocol.
    \end{corollary}
     We note that the obtained regret approaches the optimal regret $\Omega(\sqrt{KMT\log(|\Fc|)/\log(K)})$ of a single agent playing for $MT$ rounds \citep{agarwal2012contextual} up to logarithmic factors, which demonstrates the \emph{statistical efficiency} of the proposed FedIGW. Moreover, the total $O(\log(T))$ times call of the FL protocol indicates that only a limited number of agents-server information-sharing are required, which further illustrates its \emph{communication efficiency}.
    
    As the finite function class is not often practically useful, we then focus on the canonical FL setting that each $f\in \Fc$ is parameterized by a $d$-dimensional parameter $\omega \in \Rb^{d}$ as $f_{\omega}$, e.g., a neural network.
    To facilitate discussions, we abbreviate $\Sc: = \Sc_{[M]}$ while denoting $\omega^*_{\Sc} := \argmin_{\omega}\widehat{\Lc}(f_{\omega}; \Sc)$ as the empirical optimal parameter given a fixed dataset $\Sc$ and $\widehat{\omega}_{\Sc}$ as the output of the adopted FL protocol. We further assume $f^*$ is parameterized by the true model parameter $\omega^*$, and for a fixed $\omega$, define $\Lc(f_\omega):= \Eb_{\Sc}[\widehat{\Lc}(f_{\omega}; \Sc)]$ as its expected loss w.r.t. the data distribution. 
    
    Following standard learning-theoretic analyses, the key task excess risk $\Ec(\Fc; n_{[M]})$ can be bounded via a combination of errors stemming from optimization and generalization.
    \begin{lemma}\label{lem:error_decompose}
        If the loss function $l_m(\cdot; \cdot)$ is $\mu_f$-strongly convex in its first coordinate for all $m\in [M]$, it holds that $
            \Ec(\Fc; n_{[M]})\leq  2\left(\varepsilon_{\text{opt}}(\Fc; n_{[M]}) + \varepsilon_{\text{gen}}(\Fc; n_{[M]})\right)/\mu_f$,
        where $\varepsilon_{\text{gen}}(\Fc; n_{[M]}):= \Eb_{\Sc,\xi}[\Lc(f_{\widehat{\omega}_{\Sc}})- \widehat{\Lc}(f_{\widehat{\omega}_{\Sc}}; \Sc)]$ and $\varepsilon_{\text{opt}}(\Fc; n_{[M]}):= \Eb_{\Sc, \xi}[\widehat{\Lc}(f_{\widehat{\omega}_{\Sc}}; \Sc) - \widehat{\Lc}(f_{\omega^*_{\Sc}}; \Sc)]$.
    \end{lemma}
    For the generalization error term $\varepsilon_{\text{gen}}(\Fc; n_{[M]})$, we can utilize standard results in learning theory (e.g., uniform convergence). For the sake of simplicity, we here leverage a distributional-independent upper bound on the Rademacher complexity, denoted as $\Rf(\Fc; n_{[M]})$ (rigorously defined in Eqn.~\eqref{eqn:rademacher}), which provides that $\varepsilon_{\text{gen}}(\Fc; n_{[M]}) \leq 2\Rf(\Fc; n_{[M]})$ using the classical uniform convergence result (see Lemma~\ref{lem:rademacher}). We do not further particularize this upper bound while noting it can be specified following standard procedures \citep{mohri2018foundations, bartlett2005local}.

    On the other hand, the optimization error term $\varepsilon_{\text{opt}}(\Fc; n_{[M]})$ is exactly the standard convergence error in the analysis of FL protocols. Thus, once any theoretical breakthrough on the convergence of one FL protocol is reported, the obtained result can be immediately incorporated into our analysis framework to characterize the performance of FedIGW using that FL protocol. In particular, the following corollary is established to demonstrate the \emph{modularized plug-in} of analyses of different FL protocols, where FedAvg \citep{mcmahan2017communication} and SCAFFOLD \citep{karimireddy2020scaffold} are adopted as further specific instances. To the best of our knowledge, this is the first time that convergence analyses of FL protocols can directly benefit the analysis of FCB designs.
    \begin{corollary}[Modularized Plug-in of FL Analyses; A Simplified Version of Corollary~\ref{col:convex_raw_full}]\label{col:convex_raw}
    Under the condition of Lemma~\ref{lem:error_decompose}, the regret of FedIGW can be bounded as
        \begin{align*}
            \Reg(T) = O\left(ME^1 + \sum_{l\in [2,l(T)]}\sqrt{K\left(\Rf^{l-1} + \varepsilon_{\text{opt}}^l)\right)/\mu_f}ME^l\right),
        \end{align*}
        where $\Rf^{l}: =\Rf(\Fc; \{E^l: m \in [M]\})$ and using $\rho^l$ rounds of communications (i.e., global aggregations) and $\kappa^l$ rounds of local updates in epoch $l$, under a few other standard conditions,
        \begin{itemize}[leftmargin=*, nolistsep, topsep=0pt]
            \item with \textbf{FedAvg} as the adopted $\texttt{FLroutine}(\cdot)$, it holds that $\varepsilon_{opt}^l \leq \tilde{O}((\rho^l \kappa^l M)^{-1} + (\rho^l)^{-2})$;
            \item with \textbf{SCAFFOLD} as the adopted $\texttt{FLroutine}(\cdot)$, it holds that $\varepsilon_{opt}^l \leq \tilde{O}((\rho^l \kappa^l M)^{-1})$.
        \end{itemize}
    \end{corollary}
    From this corollary, we can see that FedIGW enables a general analysis framework to seamlessly leverage theoretical advances in FL, in particular, convergence analyses. Thus, besides FedAvg and SCAFFOLD, when switching the FL component in FedIGW to FedProx \citep{li2020federated}, FedOPT \citep{reddi2020adaptive}, and other existing or forthcoming FL designs, we can effortlessly plug in their optimization errors to obtain corresponding performance guarantees of FedIGW. This convenience highlights the theoretically intimate relationship between FedIGW and canonical FL studies.
    
   Moreover, Corollary~\ref{col:convex_raw} can also guide how to perform the adopted FL protocol. As the generalization error is an inherent property that cannot be bypassed by better optimization results, there is no need to further proceed with the iterative FL process as long as the optimization error does not dominate the generalization error, which is reflected in a more particularized corollary in Corollary~\ref{col:convex}.

    \begin{remark}[A Linear Reward Function Class]\label{rmk:linear}\normalfont
    As a more specified instance, we consider linear reward functions as in federated linear bandits, i.e., $f_{\omega}(\cdot) = \langle \omega,  \phi(\cdot)\rangle$ and $f^*(\cdot) = \langle \omega^*,  \phi(\cdot)\rangle$, where $\phi(\cdot) \in \Rb^{d}$ is a known feature mapping. In this case, the FL problem can be formulated as a standard ridge regression with $\ell_m(f_\omega(x_m, a_m); r_m): = \left(\langle \omega, \phi(x_m, a_m)\rangle - r_m\right)^2 + \lambda \|\omega\|_2^2$. With a properly chosen regularization parameter $\lambda = O(1/n)$, the generalization error can be bounded as $\varepsilon_{\text{gen}}(n_{[M]})= \tilde{O}(d/n)$ \citep{hsu2012random}, while a same-order optimization error can be achieved by many efficient distributed algorithms \citep{nesterov2003introductory} with roughly $O(\sqrt{n}\log(n/d))$ rounds of communications. Then, with an exponentially growing epoch length, FedIGW can have a regret of $\tilde{O}(\sqrt{dMKT})$ with at most $\tilde{O}(\sqrt{MT})$ rounds of communications as illustrated in Appendix~\ref{subapp:linear}, both of which are efficient with sublinear dependencies on the number of agents $M$ and time horizon $T$. It is worth noting that during this process, no raw or compressed data is communicated -- only processed model parameters (e.g., gradients) are exchanged. This aligns with FL studies while is distinctive from previous designs for federated linear bandits \citep{dubey2020differentially,li2022asynchronous,he2022simple,fan2023federated}, which often communicate covariance matrices or aggregated rewards.
    \end{remark} 

    {
    \begin{remark}[Beyond Linear Reward Functions]\label{rmk:gen_linear}\normalfont 
    This modularized framework can be further adopted in analyzing other reward functions as long as the corresponding excess risks can be provided. For example, the optimization errors of FedAvg (and its variants) with neural networks as the function class can be obtained from many recent works, including Theorem 4.1 in \citet{huang2021fl} and Theorem 1 in \citet{song2023fedavg}. The corresponding generalization error can also be established following existing results, e.g., Theorem 4.3 in \citet{huang2021fl} and Chapter 11 in \citet{zhang_2023}. Combining these two parts of analyses can lead to bounds on regret and communication rounds that are sublinear in $M$ and $T$ using the analysis framework.
    \end{remark}
    }

    \begin{figure}[thb]
        \centering
        \subfigure{
        \includegraphics[width=0.48\linewidth]{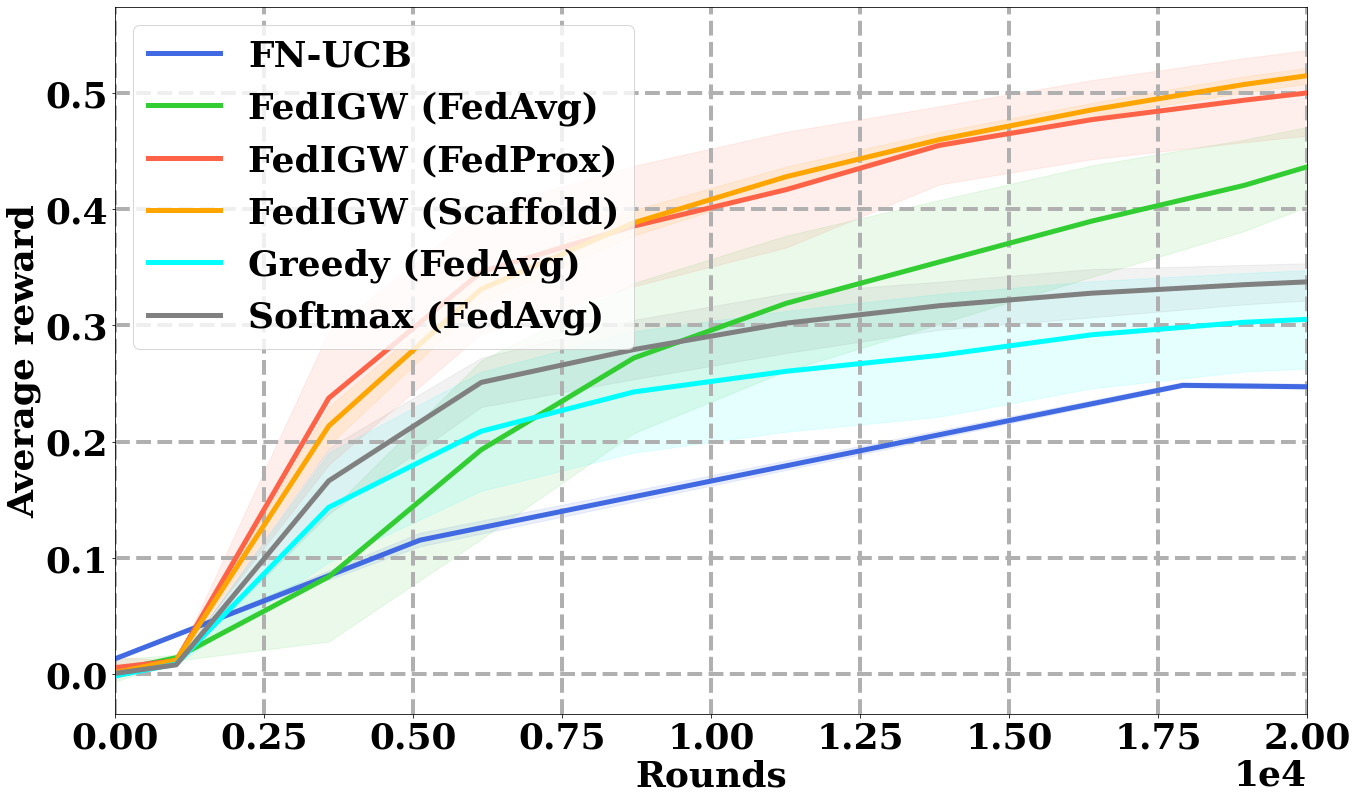}}
        \subfigure{
        \includegraphics[width=0.48\linewidth]{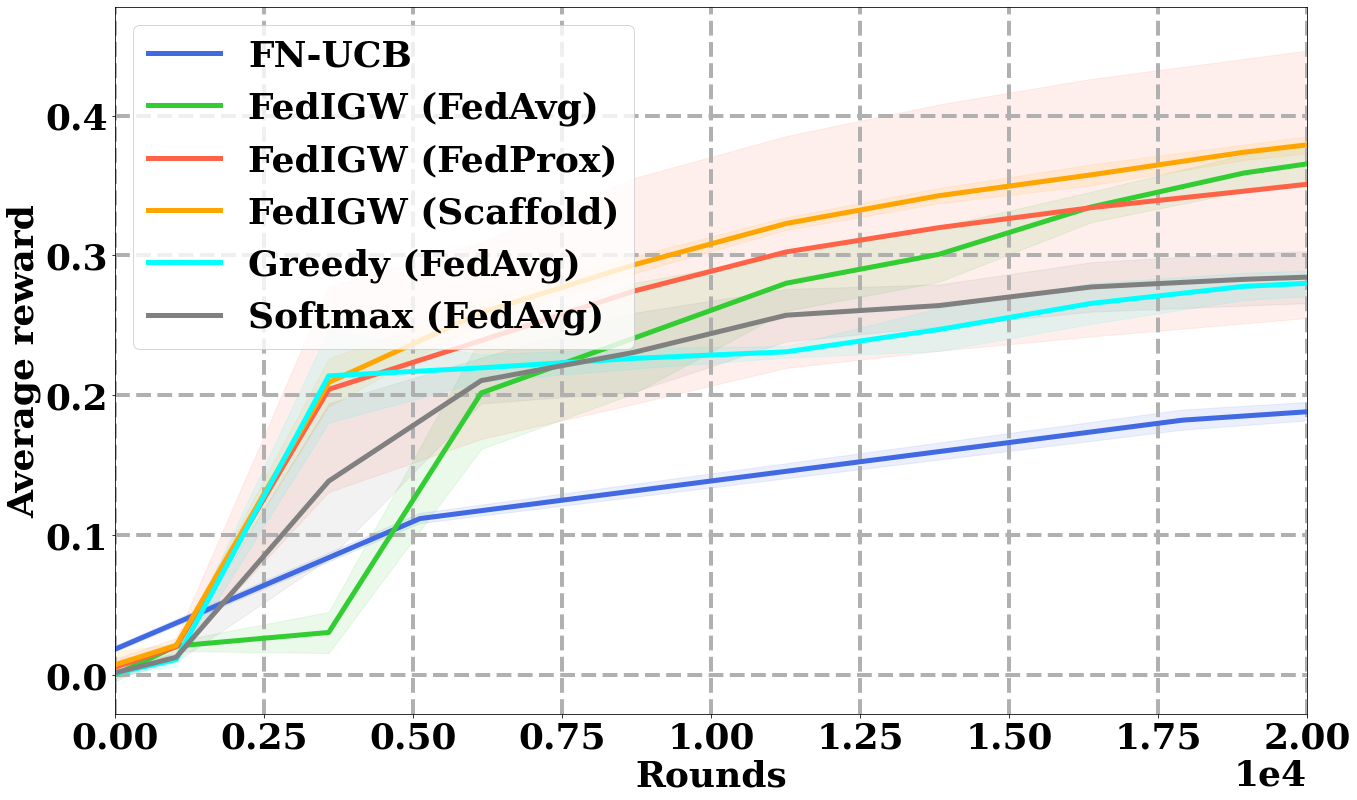}}
        \caption{\centering The averaged reward collected by each agent via FedIGW (using different FL protocols), the state-of-the-art FN-UCB, {and two other naive baselines (i.e., greedy and softmax using FedAvg)} with $M = 10$ participating agents on Bibtex (left) and Delicious (right) datasets.}
        \label{fig:exp_main}
        \vspace{-0.2in}
    \end{figure}
    
    \section{Experimental Results} \label{sec:exp} 
    In this section, we report the empirical performances of FedIGW on two distinct real-world multi-label classification datasets, Bibtex \citep{katakis2008multilabel} and Delicious \citep{tsoumakas2008effective}, which are also used in other practical CB investigations such as \cite{cortes2018adapting}. The aim of CB in these experiments is considered to be recommending one of the correct labels at any given time. Especially, in the experiments, at each time step, a context is randomly sampled from the dataset while the true labels are concealed from the agents. The agents then determine which label to select (i.e., pull one arm) with their CB algorithms; thus, the number of arms is the number of possible labels in each dataset. Upon pulling one arm, a reward of $1$ is granted if the pulled arm corresponds to one of the true labels, while a reward of $0$ is granted otherwise. From Table~\ref{tab:tasks}, we can observe that these tasks are challenging given their high-dimensional contexts ($>500$) and large numbers of arms ($>150$). Additional experimental details and results are discussed in Appendix~\ref{app:exp}, while the codes for the experiments can be found at \url{https://github.com/ShenGroup/FedIGW}.

    \begin{wraptable}{r}{0.5\textwidth}
        \vspace{-0.2in}
        \centering
        \caption{\centering The context dimension and number of arms in Bibtex and Delicious}
        \begin{tabular}{c|c|c}
        \hline
          Task   &  Context dimension & Number of arms  \\ \hline
            Bibtex  & 1835 & 159 \\
            Delicious & 500 & 983 \\
        \hline
        \end{tabular}
        \label{tab:tasks}
    \end{wraptable}

    \textbf{Varying FL choices.} The reported Fig.~\ref{fig:exp_main} first compares the averaged rewards collected by each agent with FedIGW using different FL choices, including FedAvg \citep{mcmahan2017communication}, SCAFFOLD \citep{karimireddy2020scaffold}, and FedProx \citep{li2020federated}. This is the first time, to the best of our knowledge, that FedAvg is practically integrated with FCB experiments, let alone other FL protocols, which largely demonstrate the generality and flexibility of FedIGW. It can be observed that using the more developed SCAFFOLD and FedProx provides improved performance (i.e., collects more rewards) compared with the basic FedAvg, 
    which credits to that FedIGW can flexibly leverage algorithmic advances in FL protocols.

    {\textbf{Comparison with baselines.} To further evaluate the performance of FedIGW, experiments are conducted to compare it with
    several baselines as described in the following. 
    \begin{itemize}[noitemsep,topsep=0pt,leftmargin = *]
        \item \textbf{FN-UCB \citep{dai2023federated}.} The federated neural-upper confidence bound (FN-UCB) design proposed in \cite{dai2023federated} is adopted as a strong FCB baseline due to its capability of leveraging neural networks to approximate rewards and the previously reported good performance. Instead of being compatible with canonical FL protocols, FN-UCB requires a specifically developed communication design, where local neural tangent features are transmitted to the server for global aggregation in a one-shot fashion.
        \item \textbf{Greedy and softmax.} Besides IGW, two other regression-based CB algorithms, greedy selection and softmax selection, are also adopted for empirical validations using FedAvg to collaboratively learn the reward function. In particular, the action is selected as $a_{m,t_m} \gets \argmax_{a_m\in \Ac_m} \widehat{f}^{l}(a_m, x_{m,t_m})$ for greedy and $a_{m,t_m}\sim \text{softmax}(\widehat{f}^{l}(\cdot, x_{m,t_m})/\zeta)$ for softmax, where $\zeta$ is a tempurate parameter.
    \end{itemize}
    In Fig.~\ref{fig:exp_main}, all methods leverage the same-size MLPs to approximate reward functions for fair comparisons. It can be observed that after convergence, FedIGW (even with the basic FedAvg) significantly outperforms FN-UCB with about twice the rewards collected by each agent on average, demonstrating its remarkable superiority. Also, under the FL protocol (i.e., FedAvg), FedIGW exhibits much stronger performance than greedy and softmax, further illustrating the advantage of using IGW as the CB algorithm. 
    }

    \section{Flexible Extensions: Seamless Integration of FL Appendages}\label{sec:appendage}
    Another notable advantage offered by the flexible FL choices is to bring appealing appendages from FL studies to directly benefit FCB. In the following, we discuss how to leverage techniques of personalization, robustness, and privacy from FL in FedIGW while presenting intriguing avenues for future exploration.
    
    \subsection{Personalized Learning}\label{subsec:per}
    In many cases, each agent's true reward function is not globally realizable as in Assumption~\ref{asp:realizable}, but instead only locally realizable in her own function class as in the following assumption.
    \begin{assumption}[Local Realizability]\label{asp:local}
    For each $m\in [M]$, there exists $f^*_m$ in $\Fc_m$ such that $f^*_m(x_m, a_m) =  \mu_m(x_m, a_m)$ for all $x_m \in \Xc_m$ and $a_m \in \Ac_m$
    \end{assumption}
    Following discussions in Sec.~\ref{subsec:concrete}, we consider that each function $f$ in $\Fc_m$ is parameterized by a $d_m$-dimensional parameter $\omega_m \in \Rb^{d_m}$, which is denoted as $f_{\omega_m}$. Correspondingly, the true reward function $f^*_m$ is parameterized by $\omega^*_m$ and denoted as $f_{\omega^*_m}$.
    To still motivate the collaboration and motivated by popular personalized FL studies \citep{hanzely2021personalized, agarwal2020federated}, we study a middle case where only partial parameters are globally shared among $\{f_{\omega_m^*}: m\in [M]\}$ while other parameters are potentially heterogeneous among agents, which can be formulated via the following assumption. 
    \begin{assumption}\label{asp:per}
        For all $m\in [M]$, the true parameter $\omega^*_{m}$ can be decomposed as $[\omega^{\alpha,*}, \omega^{\beta,*}_m]$ with $\omega^{\alpha,*}\in \Rb^{d^\alpha}$ and $\omega^{\beta,*}_m \in \Rb^{d^{\beta}_m}$, where $d^{\alpha} \leq \min_{m\in [M]} d_m$ and $d^{\beta}_m := d_m - d^{\alpha}$. In other words, there are $d^\alpha$-dimensional globally shared parameters among $\{\omega^*_m:m\in [M]\}$.
    \end{assumption}
    A similar setting is studied in \cite{li2022asynchronous} for linear reward functions and in \cite{agarwal2020federated} for realizable cases with a naive $\varepsilon$-greedy design for CB. For FedIGW, we can directly adopt a personalized FL protocol (such as LSGD-PFL in \citet{hanzely2021personalized}) to solve a standard personalized FL problem:
    \begin{equation*}
        \min_{\omega^\alpha, \omega^\beta_{[M]}} \widehat{\Lc}(f_{\omega^\alpha, \omega^\beta_{[M]}}; \Sc_{[M]}) :=\sum_{m\in [M]} (n_m/n) \cdot \widehat{\Lc}_m(f_{\omega^\alpha, \omega^\beta_m}; \Sc_m).
    \end{equation*}
    With outputs $\widehat{\omega}^\alpha$ and $\widehat{\omega}^\beta_{[M]}$, the corresponding $M$ functions $\{f_{\widehat{\omega}^\alpha, \widehat{\omega}^\beta_{m}}:m\in [M]\}$ (instead of the single one $\widehat{f}$ in Sec.~\ref{subsec:alg}) can be used by the $M$ agents, separately, for their CB interactions following the IGW algorithm. Concrete results and more details can be found in Appendix~\ref{app:per}. 

    \begin{remark}[A Linear Reward Function Class]\normalfont
    Similar to Remark~\ref{rmk:linear}, we also consider linear reward functions for the personalized setting with $f^*_m(\cdot) := \langle \omega^*_m, \phi(\cdot) \rangle$ and $\{\omega^*_m: m\in [M]\}$ satisfying Assumption~\ref{asp:per}. Then, FedIGW still can achieve a regret of $\tilde{O}(\sqrt{\tilde{d}MKT})$ with $\tilde{O}(\sqrt{MT})$ rounds of communications, where $\tilde{d}:= d^{\alpha}+ \sum_{m\in [M]}d^\beta_m$; see more details in Appendix~\ref{subapp:per_linear}.
        
    \end{remark}
    
    \subsection{Robustness, Privacy, and Beyond} \label{subsec:beyond}
    Another important direction in FCB studies is to improve robustness against malicious attacks and provide privacy guarantees for local agents. A few progresses have been achieved in attaining these desirable attributes for FCB but they typically require substantial modifications to their base FCB designs, such as robustness in \citet{demirel2022federated, jadbabaie2022byzantine, mitra2022collaborative} and privacy guarantees in \citet{dubey2020differentially,zhou2023differentially,li2022privacy,huang2023federated}. 
    
    With FedIGW, it is more convenient to achieve these attributes as suitable techniques from FL studies can be seamlessly applied. Especially, robustness and privacy protection have been extensively studied for FL in  \citet{yin2018byzantine, pillutla2022robust, fu2019attack} and  \citet{wei2020federated, yin2021comprehensive, liu2022privacy}, respectively, among other works. As long as such FL protocols can provide an estimated function (which is the default goal of FL), they can be adopted in FedIGW to achieve additional robustness and privacy guarantees in FCB; see more details in Appendix~\ref{app:beyond}.
    
    \textbf{Other Possibilities.} There have been many studies on fairness guarantees \citep{mohri2019agnostic, du2021fairness}, client selections \citep{balakrishnan2022diverse,fraboni2021clustered}, and practical communication designs \citep{chen2021distributed, wei2022federated, zheng2020design} in FL among many other directions, which are all conceivably applicable in FedIGW. In addition, \citet{marfoq2023federated} studies FL with data streams, i.e., data comes sequentially instead of being static, which is a suitable design for FCB as CB essentially provides data streams. If similar ideas can be leveraged in FCB, the two components of CB and FL can truly be parallel. 
    
    \section{Conclusions}
    In this work, we studied the problem of federated contextual bandits (FCB). It is first recognized that existing FCB designs are largely disconnected from canonical FL studies in their adopted FL protocols, which hinders the integration of crucial FL advancements. To bridge this gap, we introduced a novel design, FedIGW, capable of accommodating a wide range of FL protocols, provided they address a standard FL problem. A comprehensive theoretical performance guarantee was provided for FedIGW, highlighting its efficiency and versatility. Notably, we demonstrated the modularized incorporation of convergence analysis from FL by employing examples of the renowned FedAvg \citep{mcmahan2017communication} and SCAFFOLD \citep{karimireddy2020scaffold}. Empirical validations on real-world datasets further underscored its practicality and flexibility. Moreover, we explored how advancements in FL can seamlessly bestow additional desirable attributes upon FedIGW. Specifically, we delved into the incorporation of personalization, robustness, and privacy, presenting intriguing opportunities for future research.

    It would be valuable to pursue further exploration of alternative CB algorithms within FCB, e.g., \cite{xu2020upper, foster2020instance, wei2021non}, and investigate whether the FedIGW design can be extended to more general federated RL \citep{dubey2021provably,min2023multi}.

    \section*{Acknowledgement}
    The work of CSs and KY was partially supported by the U.S. National Science Foundation (NSF) under awards ECCS-2033671, ECCS-2143559, CPS-2313110, and CNS-2002902, and the Bloomberg Data Science Ph.D. Fellowship.

\bibliography{ref}
\bibliographystyle{tmlr}

\newpage
\appendix
\section{Additional Discussions}
\subsection{Societal Impacts}
This work focuses on providing a new design for federated contextual bandits (FCB), which establishes a close relationship between FCB and FL. We do not foresee major negative societal impacts as FCB is a well-established research domain and this work largely investigates its theoretical aspects. Moreover, as discussed in Section~\ref{subsec:beyond}, FedIGW can conveniently incorporate appendages from FL studies to obtain appealing properties of privacy, robustness, fairness, and beyond, which we believe can contribute to a positive societal impact. 

\subsection{Examples of FL Components in FCB Studies}\label{subapp:example}
An example of the FL components adopted in previous FCB studies is provided in the following, together with the renowned FedAvg protocol for comparison. Specifically, as in Remark~\ref{rmk:linear}, we consider the study of federated linear bandits with a known $d$-dimensional feature mapping $\phi(\cdot, \cdot)$. Then, Alg.~\ref{alg:fedlinear} illustrates the FL component commonly adopted in \cite{wang2019distributed,li2022asynchronous,dubey2020differentially,he2022simple}: the agents share compressed local data  (e.g., covariance matrices) to the server for aggregation, which happens in a one-shot fashion. A simplified version of FedAvg \citep{mcmahan2017communication} is presented in Alg.~\ref{alg:fedavg}, with client $m$'s local loss function denoted as $\hat{\Lc}_m(\cdot; \cdot)$ following Sec.~\ref{sec:FedIGW}. It can be observed that FedAvg takes an optimization perspective to perform multi-rounds of gradient descent distributively.
 
\begin{figure}[htb]
\begin{minipage}{0.5\textwidth}
    \begin{algorithm}[H]
    \caption{The FL component commonly adopted in existing studies on federated linear bandits: \textbf{one-shot aggregation of compressed local data}}
    \label{alg:fedlinear}
    \begin{algorithmic}[1]
        \Require $M$ clients with client $m$'s interaction dataset denoted as $\Sc_m = \{(x_{m,\tau_m}, a_{m,\tau_m}, r_{m, \tau_m}): \tau_m\in [n_m] \}$
        \State \textcolor{blue}{\texttt{Client $m$}}: with $\phi(x_{m,\tau_m}, a_{m,\tau_m})$ denoted as $\phi_{m,\tau_m}$, compute $V_{m} \gets \sum_{\tau_m \in [n_m]} \phi_{m,\tau_m}\phi_{m,\tau_m}^\top$ and $b_m \gets \sum_{\tau_m\in [n_m]}r_{m,\tau_m} \phi_{m,\tau_m}$
        \State \textcolor{blue}{\texttt{Client $m$}}: send $V_m$ and $b_m$ to the server
        \State \textcolor{magenta}{\texttt{Server}}: receive $V_m$ and $b_m$ from each client $m$
        \State \textcolor{magenta}{\texttt{Server}}: compute $V \gets \sum_{m\in [M]} V_m$ and $b\gets \sum_{m\in [M]} b_m$
        \State \textcolor{magenta}{\texttt{Server}}: send $V$ and $b$ to all clients
        \State \textcolor{blue}{\texttt{Client $m$}}: receive $V$ and $b$ from the server
    \end{algorithmic}
    \end{algorithm}
\end{minipage}
\begin{minipage}{0.5\textwidth}
    \begin{algorithm}[H]
    \caption{The (simplified) FedAvg algorithm as an example of the canonical FL framework: \textbf{multiple-round aggregation of local model parameters}}
    \label{alg:fedavg}
    \begin{algorithmic}[1]
    \Require $M$ clients with client $m$'s interaction dataset denoted as $\Sc_m = \{(x_{m,\tau_m}, a_{m,\tau_m}, r_{m, \tau_m}): \tau_m\in [n_m] \}$, learning rate $\eta$
    \For{$i = 1, 2, \cdots$}
	\State \textcolor{blue}{\texttt{Client $m$}}: update $\hat{\omega}_m' \gets \hat{\omega}_m - \nabla \hat{\Lc}_m(\hat{\omega}_m; \Sc_m)$
        \State \textcolor{blue}{\texttt{Client $m$}}: send $\Delta_m \gets \hat{\omega}_m' - \hat{\omega}_m$ to the server
        \State \textcolor{magenta}{\texttt{Server}}: receive $\Delta_m$  from each client $m$
        \State \textcolor{magenta}{\texttt{Server}}: with $\sum_{m\in [M]}n_m$ denoted as $n$, send $\hat{\omega} \gets \hat{\omega} - \eta\sum_{m\in [M]}\frac{n_m}{n}\Delta_m$ to all clients
        \State \textcolor{blue}{\texttt{Client $m$}}: receive $\hat{\omega}'$ from the server and set $\hat{\omega}_m \gets \hat{\omega}$    
    \EndFor
    \end{algorithmic}
    \end{algorithm}
\end{minipage}
\end{figure}

\subsection{Limitations and Future Works}
While this work proposes a novel, broadly applicable FCB design, i.e., FedIGW, there are still many interesting directions that are worth further exploring.

$\bullet$ \textbf{Paralleling CB and FL.} As mentioned in Section~\ref{sec:principle}, the current FL studies largely focus on learning from batched and static datasets. To accommodate such protocols, FCB designs typically follow a periodically alternating scheme as shown in Fig.~\ref{fig:principle}, which is thus the focus of this work. While such alternating designs are capable of achieving statistical and communication efficiency, there is still room for improvement: (1) the CB interactions need to wait for the completeness of a full FL process, which may be slow when computation resources are limited and communication delays are large; (2) it is desirable to use the CB data in a more timely fashion instead of accumulating to the end of an epoch.

As one variant of periodically alternating, we can have FedIGW interleave CB and FL as shown in Fig.~\ref{fig:interleaving}. This approach provides some buffer to perform FL without agents waiting for its completeness. Especially, in epoch $l$, on one hand, the agents perform FL with datasets from epoch $l-1$; on the other hand, they perform CB interactions following IGW with an estimated function $\widehat{f}^{l-2}$ learned during epoch $l-1$ via datasets from epoch $l-2$. In other words, there will be one epoch delay compared with the basic form of FedIGW, while this delay is used for the FL process.

Furthermore, a better approach is to have FL and CB fully paralleled as shown in Fig.~\ref{fig:parallel}. Then, neither of them needs to wait for the other part, while CB data can be processed more timely. As mentioned in Section~\ref{subsec:beyond}, we believe that the framework of FL with data streams proposed in a recent work of \citet{marfoq2023federated} could be a suitable tool, as the sequential CB interactions essentially provide data steams. We believe this direction is not only worth further exploring in FCB but perhaps more importantly, calls for more investigation in FL with data streams, where FCB can also serve as an important motivation application.

\begin{figure}[htb]
    \centering
        \subfigure[Interleaving]{\includegraphics[width=0.45\linewidth]{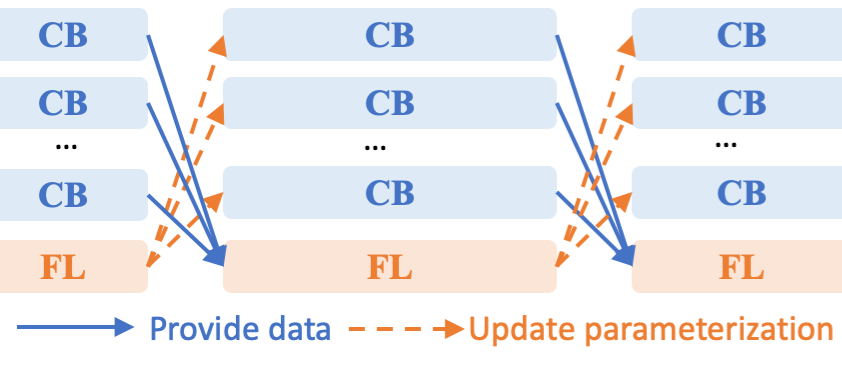}\label{fig:interleaving}}
        \hspace{0.2in}
        \subfigure[Fully Paralleling]{\includegraphics[width=0.45\linewidth]{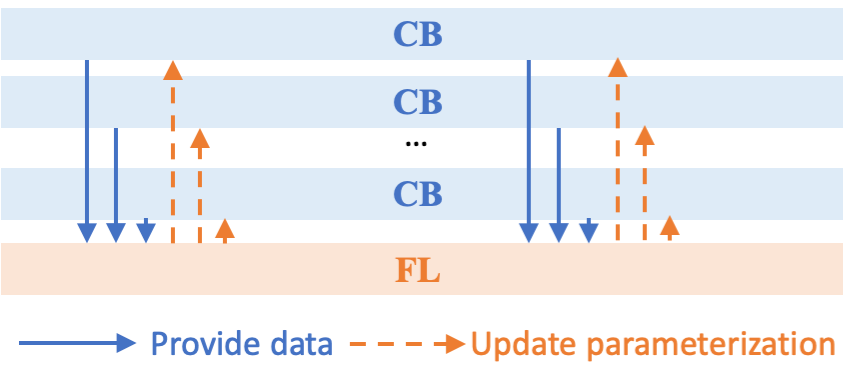}\label{fig:parallel}}
        \caption{\centering Different Styles of Connecting FL and CB in FCB.}
\end{figure}

$\bullet$ \textbf{Incorporating other FL advances.} Given the flexible FL choice in FedIGW, although this work has provided detailed discussions on incorporating many aspects of FL advancements (including canonical algorithmic designs, convergence analysis, and useful appendages), there are still many directions worth further exploration. For example, as mentioned in Section~\ref{subsec:formulation}, this work and most FCB investigations are focused on collaborating through a central server, while the case of communicating via a connected graph is less explored, where certain consensus errors commonly appear \citep{xin2020variance,ye2020pmgt}. It is worth noting that the design and analysis framework of FedIGW are both applicable in the later setting. Especially, the consensus error can be modeled as one part of the optimization error in Lemma~\ref{lem:error_decompose}. This further validates the value of the proposed FedIGW design and the general analysis framework while further specifications are left for future works.

Also, it would be great to leverage extra tools to save computations in the adopted FL protocol. Using local updates as in \citet{chou2020general} is one promising direction. These approaches are all feasible in FedIGW as long as the agents can obtain a learned reward function to perform IGW interactions. Their specific impacts can be captured via the established analysis framework through their own optimization errors.

$\bullet$ \textbf{Leveraging other CB designs.} With previous FCB studies largely focused on the CB component, this work is motivated to incorporate more advances from FL. Thus, we propose the FedIGW design which can leverage canonical protocols, convergence analyses, and flexible appendages from FL.

However, we also note that there are still many CB algorithms that remain under-explored in FCB, where UCB-based designs are dominating. For example, the simple greedy algorithm is shown to be efficient when the context generation contains certain exploration capabilities in \citep{han2020sequential}. Moreover, varying attempts have been made in \cite{xu2020upper, foster2020beyond, foster2020instance, zhu2022contextual} to design generally applicable CB algorithms with tight performance guarantees, e.g., handling infinite arms. It would be interesting to investigate how to bring these designs to the federated setting and whether such connections provide new opportunities and insights.

$\bullet$ \textbf{Complex environments.} This work is focused on a stationary environment with stochastic rewards, which is well motivated by practical applications and commonly adopted in FCB studies. To further broaden the applicability of FCB, we believe that it is also important to study adversarial or non-stationary environments. Many advances have been made in standard single-agent bandits, e.g., \cite{auer2002nonstochastic,neu2020efficient, zierahn2023nonstochastic, wei2021non}. A recent work \citep{yi2023doubly} investigates the federated adversarial environment in the tabular setting and further investigations are desired to provide further concrete designs and analyses.

$\bullet$ \textbf{Extension to RL.} It would also be meaningful to extend the current study of FCB to federated reinforcement learning (RL) as a further step in understanding the combination of FL and sequential decision-making. Some results have been reported in \cite{dubey2021provably,min2023multi,jin2022federated,fan2021fault, cisneros2023one}. We hope this work can serve as a starting point for more principled and generally applicable studies in federated RL.

\section{Additional Related Works}\label{app:related}
The studies on federated multi-armed bandits (FMAB) and federated contextual bandits (FCB) can be viewed as a version of the general multi-agent bandits  \citep{liu2010distributed, boursier2019sic,shi2020decentralized,shi2021heterogeneous} and parallelizing bandits \citep{chan2021parallelizing, karbasi2021parallelizing} that is more suitable for modern applications. We provide a more detailed review in the following.

$\bullet$ \textbf{Tabular.} There have been many studies on cooperative designs in multi-armed bandits (i.e., the tabular setting), e.g., \cite{hillel2013distributed, szorenyi2013gossip,landgren2016distributed,martinez2019decentralized}, focusing on different learning targets and different communication protocols (e.g., through a communication graph or with some randomly selected peers). Notably, in \cite{wang2019distributed}, communication-efficient designs are proposed via periodically aggregating local estimates and performing arm elimination globally. We here also discuss another line of works on FMAB \citep{shi2021federated,shi2021personal,reda2022near,zhu2021federated,chen2022federated,shi2023reward}. In their considered setting, the global rewards are (weighted) averages of local observations; however the former is not directly observable. With maximizing global rewards as the learning target, the agents need to collaboratively perform explorations and aggregate local information. Despite the model differences, the design principle of FedIGW may still be beneficial for studying this setting. Especially, it is worth considering replacing the UCB-based explorations commonly adopted in \citet{shi2021federated,shi2021personal,reda2022near,zhu2021federated,chen2022federated} with regression-based ones as in FedIGW to facilitate incorporation of FL studies.

$\bullet$  \textbf{Linear.} The most commonly studied FCB setting is federated linear bandits. There have been many investigations in this direction.  Especially, different environments have been tackled in different works, e.g., the finite-armed fixed-context setting \citep{wang2019distributed,huang2021federated}, the finite-armed stochastic-context setting \citep{amani2022distributed}, the finite-armed adversarial context setting \citep{fan2023federated}, the infinite-armed fixed-context setting \citep{salgia2022distributed}, and the infinite-armed adversarial-context setting \citep{wang2019distributed,dubey2020differentially,li2022asynchronous, he2022simple}. Furthermore, many other settings, e.g., unobserved context \citep{lin2023federated}, and additional properties, e.g., privacy \citep{dubey2020differentially,zhou2023differentially}, robustness \citep{jadbabaie2022byzantine}, have been investigated. As summarized in the main paper, these works mainly select arm elimination (AE) \citep{lattimore2020bandit} or LinUCB \citep{abbasi2011improved} as their CB designs, which require both model estimates and confidence bounds. Thus, in their designed FL protocols, compressed local data (e.g., aggregated local rewards and covariance matrices) are often directly shared to solve a global ridge regression and to construct tighter confidence bounds. Compared with these studies, FedIGW can effectively solve the finite-armed stochastic-context setting without sharing any raw or compressed local data but only communicate processed model parameters (e.g., gradients). More detailed discussions and concrete results are provided in Appendix~\ref{subapp:linear}.

A detailed comparison of the obtained regrets and the amounts of communicated real numbers is provided in Table~\ref{tab:linear}. It can be observed that adapting FedIGW to the specific case of linear bandits does not provide the same near-optimal performance as in previous works. This is not a surprise as (single-agent) IGW itself has not yet been shown to achieve the lower-bound performance of linear bandits, while the previous works are largely built upon the nearly optimal LinUCB design \citep{abbasi2011improved}. However, as noted in Remark~\ref{rmk:IGW}, IGW only requires a learned reward function, instead of complicated data analytics such as UCB, which grants it great flexibility to better incorporate FL advancements and handle more general scenarios beyond the linear setting. 

$\bullet$  \textbf{Generalized Linear and Kernelized.} As extensions of the linear reward functions, \cite{li2022generalized} considers the generalized-linear class, and \cite{li2022kernelized, li2023learning} study the kernelized one. The adopted basic techniques are similar to the aforementioned ones in federated linear bandits, while efforts are focused on fine-tuning communications (e.g., via Nystr\"om approximation \citep{li2022kernelized, li2023learning}). It is worth noting that \cite{li2022generalized} invokes the distributed accelerated gradient descent algorithm to solve their considered distributed optimization with a generalized linear function class, which can be viewed as a preliminary attempt of involving FL or distributed optimization designs in FCB. However, the motivation there is the lack of a closed-form solution as in the linear case, while \cite{li2022generalized} additionally needs to share the local covariance matrices to construct better confidence bounds. This work, instead, formally proposed FedIGW which can rely only on canonical FL framework and accommodate flexible FL choices.

$\bullet$ \textbf{Neural.} A recent work of \citet{dai2023federated} extends the advances on single-agent neural bandits \citep{zhou2020neural} to the federated setting, where the neural tangent kernel (NTK) analyses are incorporated. With NTK to ``linearize'' the considered over-parameterized neural network, \cite{dai2023federated} still largely follows the designs in the aforementioned federated linear bandits while some additional attempts have been made, e.g., an extra one-round averaging of model parameters besides aggregating NTK. This work, instead, takes a step further to fully leverage FL protocols, which often perform multiple (instead of one) rounds of model aggregations that are often necessary to guarantee convergence. Also, the optimization and generalization errors of a FedAvg variant with overparameterized neural networks are provided in \cite{huang2021fl}, which is conceivably compatible with FedIGW for the corresponding analyses. Moreover, as shown by the additional experimental results in Sec.~\ref{sec:exp}, FedIGW empirically outperforms FN-UCB \citep{dai2023federated} on different tasks and is more computationally efficient.

\begin{table}[tb]
    \caption{\centering A comparison of settings and results of federated linear bandits; note that FedIGW is not specifically designed and optimized to handle linear reward functions as previous designs.}
    \centering
    \begin{tabular}{c|c|c|c|c}
        \hline
       Reference  & Arms & Context & Regret & \# of Numbers Communicated\\
       \hline
       \cite{wang2019distributed} & Infinite & Fixed & $\tilde{O}(d\sqrt{MT})$ & $O((dM+d\log\log(d))\log(T))$\\
       \cite{he2022simple} & Infinite & Adversarial & $\tilde{O}(d\sqrt{MT})$ &  $O(d^3M^2\log(MT))$\\
       \cite{huang2021federated} & Finite & Fixed & $\tilde{O}(\sqrt{dMT})$ & $O(d^2+ dK)M\log(T))$ \\
       \cite{amani2022distributed}$^\dagger$ & Finite & Stochastic & $\tilde{O}(\sqrt{dMT})$ & $O(dM\log\log(MT))$\\
       FedIGW$^\ddagger$ & Finite & Stochastic & $\tilde{O}(\sqrt{dKMT})$ & $O(d^2M\log(T))$\\
       FedIGW$^\flat$ & Finite & Stochastic & $\tilde{O}(\sqrt{dKMT})$ & $O(d\log(d)\sqrt{M^3T})$\\
       \hline
    \end{tabular}
    \begin{center}
        $\dagger$: assuming a homogeneous and known context distribution for all agents;\\
        $\ddagger$: solving the global ridge regression via directly sharing aggregated local rewards and covariance matrices as in the other listed works;\\
        $\flat$: solving the global ridge regression via distributed accelerated gradient descent;\\
    \end{center}
    \label{tab:linear}
\end{table}

\section{Proofs for Section~\ref{subsec:general}}
\subsection{Notations}
We first introduce notations that are repeatedly used.
For the output function from the adopted FL protocol, we characterize its performance via the following definition of its excess risk, which is commonly adopted in the analysis of IGW-type CB algorithms \citep{simchi2022bypassing, sen2021top, ghosh2021model}.
    \begin{definition}\label{def:error_general}
        Let $p_{[M]}: = \{p_m: m\in [M]\}$ be a set of $M$ arbitrary independent arm selection distributions. Given an overall dataset $\Sc_{[M]} := \{\Sc_{m}: m\in [M]\}$ where each dataset $\Sc_m$ consists of $n_m$ training samples of the form  $(x_{m}, a_{m}; r_{m}(a_{m}))$ independently and identically drawn according to $(x_{m}, r_{m})\sim \Dc_m$, $a_{m}\sim p_{m}(\cdot|x_{m})$, the federated protocol $\textup{\texttt{FLroutine}}(\Sc_{[M]}) = \{\textup{\texttt{FLroutine}}_m(\Sc_m): m\in [M]\}$ returns a predictor $\widehat{f}(\cdot)$, and its excess risk is defined as
        \begin{align*}
             \Ec(\Fc; n_{[M]}) := \Eb_{S_{[M]}, \xi}\left[\sum_{m\in [M]} \frac{n_m}{n}\cdot \Eb_{x_m\sim \Dc^{\Xc_m}_{m}, a_m\sim p_m(\cdot|x_m)}\left[\left(\widehat{f}(x_m, a_m) - f^*(x_m, a_m)\right)^2\right]\right],
        \end{align*}
        where $n_{[M]}:=\{n_m:m\in [M]\}$ and $\xi$ denotes the random source in the potentially stochastic FL algorithm. We often abbreviate $\Ec(\Fc; n_{[M]})$ as $\Ec(n_{[M]})$ to simplify notations.
    \end{definition}
    This definition measures in expectation (w.r.t. the random data generation and the stochastic FL process) how far the output of the adopted FL protocol is from the true reward function on the weighted data distribution of all agents. Note that the excess risk bound $\Ec(n_{[M]})$ would typically rely on some other parameters in the adopted FL protocol (e.g., the step size and the number of iterations in gradient-based approaches), which are currently not specified for generality.

Then, let $\Upsilon^l$ denote the sigma-algebra generated by the history up to epoch $l$, i.e., $\{(x_{m,t_m}, a_{m,t_m}, r_{m,t_m}): m\in [M], t_m\in [t_m(\tau^l)]\}$, and the randomness in the adopted FL protocol up to epoch $l$, i.e., $\{\xi_{i}: i\in [l]\}$, where $\xi_i$ denotes the random source in epoch $i$. Then, we denote $l_m(t_m): = \min\{l\in \Nb: t_m \leq t_m(\tau^l)\}$ as the epoch that agent $m$'s $t_m$ belongs to. 
Also, let $\Psi_m := \Ac_m^{\Xc_m}$ denote the set of deterministic functions from $\Xc_m$ to $\Ac_m$ for agent $m$ and $\Psi_{[M]} := \times_{m\in [M]} \Psi_m$ the Cartesian product of $\{\Psi_m: m\in [M]\}$. Furthermore, for any action selection kernel $p_{[M]} = \{p_m: m\in [M]\}$, where $p_m(a_m|x_m)$ is the probability of selecting action $a_m \in \Ac$ given context $x_m$, and any policy $\pi_{[M]} = \{\pi_m: m\in [M]\}\in \Psi$, we define
\begin{align*}
    V_m(p_m, \pi_m) &:= \Eb_{x_m\sim \Dc^{\Xc_m}_m}\left[\frac{1}{p_m(\pi_m(x_m)|x_m)}\right],\\
    \Rc_m(\pi_m) &:= \Eb_{x_m \sim \Dc_{m}^{\Xc_m}} \left[f^*(x_m, \pi_m(x_m))\right],\\
    \widehat{\Rc}_{m}^l(\pi_m \mid \Upsilon^{l-1}) &:= \Eb_{x_m \sim \Dc_{m}^{\Xc_m}}\left[\widehat{f}^{l}(x_m, \pi_m(x_m))\mid \Upsilon^{l-1}\right],\\
    \Reg_m(\pi_m) &: = \Rc_m(\pi^*_m) - \Rc_m(\pi_m),\\
    \widehat{\Reg}^l_{m}(\pi_m\mid \Upsilon^{l-1})&: = \widehat{\Rc}^l_{m,t_m}(\widehat{\pi}^{l}_m\mid \Upsilon^{l-1}) - \widehat{\Rc}^l_{m,t_m}(\pi_m\mid \Upsilon^{l-1}).
\end{align*}
where $\widehat{\pi}^{l}_m(x_m) := \argmax_{a_m\in \Ac_m} \widehat{f}^l(x_m,a_m)$ for a given $\widehat{f}^l$ (determined by $\Upsilon^{l-1}$).

The following proofs are largely inspired by the single-agent contextual bandits work \citep{simchi2022bypassing}, while major changes have been made to accommodate the more complex federated system considered in this work.

\subsection{Proofs of Theorem~\ref{thm:global_regret}}
First, the following lemma characterizes the relation between the excess errors and the selected learning rates.
\begin{lemma}\label{lem:learning_rate}
For all $l>1$, it holds that
\begin{align*}
    &\Eb_{\Upsilon^{l-1}}\left[\sum_{m\in [M]} \frac{E^{l-1}_m}{\sum_{m'\in [M]}E^{l-1}_{m'}}\cdot \Eb_{x_m\sim \Dc^{\Xc_m}_{m}, a_m\sim p^{l-1}_m(\cdot|x_m)}\left[\left(\widehat{f}^l(x_m, a_m) - f^*(x_m, a_m)\right)^2 \mid \Upsilon^{l-1}\right]\right] \\
   & \leq \Ec(\Fc; E^{l-1}_{[M]}) = \frac{\sum_{m\in [M]}E^{l-1}_m K_m}{\sum_{m\in [M]}E^{l-1}_m(\gamma^l)^2}.
\end{align*}
\end{lemma}
\begin{proof}
    The first inequality is from the Assumption~\ref{def:error_general}, while the second is based on the choice of $\gamma^l$ in Theorem~\ref{thm:global_regret}, i.e.,
    \begin{align*}
        \gamma^ l = \sqrt{\frac{\sum_{m\in [M]}E^{l-1}_m K_m}{\sum_{m\in [M]}E^{l-1}_m\Ec(\Fc; E^{l-1}_{[M]})}}, 
    \end{align*}
    which leads to the lemma.
\end{proof}

Then, the following lemma bounds the estimated rewards $\widehat{\Rc}_m^l$ and true rewards $\Rc_m$. 
\begin{lemma}\label{lem:VE}
    For any epoch $l>1$, for any $\pi_m\in \Psi_m$, conditioned on $\Upsilon^{l-1}$, it holds that
    \begin{align*}
         \left|\widehat{\Rc}_{m}^l(\pi_m\mid \Upsilon^{l-1}) - \Rc_m(\pi_m)\right| &\leq  \sqrt{V_{m}(p^{l-1}_m, \pi_m\mid \Upsilon^{l-1}) } \sqrt{\Ec_m^{l-1}(\Upsilon^{l-1})},
    \end{align*}
\end{lemma}
where $\Ec_m^{l-1}(\Upsilon^{l-1}) := \Eb_{x_m\sim \Dc^{\Xc_m}_m, a^{l-1}_m\sim p_m^{l-1}(\cdot|x_m)}\left[\left(\widehat{f}^{l}(x_{m}, a_{m}^{l-1}) - f^*(x_{m}, a_{m}^{l-1}) \right)^2\mid \Upsilon^{l-1}\right]$. 
\begin{proof}
    For simplicity, we abbreviate $\Eb_{x_m\sim \Dc^{\Xc_m}_m, a^{l-1}_m\sim p_m^{l-1}(\cdot|x_m)}[\cdot]$ as $\Eb_{x_m, a^{l-1}_m}[\cdot]$, and for any policy $\pi_m\in \Psi_m$, and any epoch $l>1$, we define
    \begin{align*}
        \Delta^{l}_{m}(\pi_m(x_m)):= \widehat{f}^{l}(x_m, \pi_m(x_m)) - f^*(x_m, \pi_m(x_m)) 
    \end{align*}
    which indicates that
    \begin{align*}
        \widehat{\Rc}^l_{m}(\pi_m\mid \Upsilon^{l-1}) - \Rc_m(\pi_m) = \Eb_{x_m} \left[\Delta^{l}_{m}(\pi_m(x_m)\mid \Upsilon^{l-1}\right],
    \end{align*}
    and
    \begin{align*}
        \Eb_{x_m, a_{m}^{l-1}}\left[\left(\Delta^{l}_{m}(a_{m}^{l-1})\right)^2\mid \Upsilon^{l-1}\right]\geq  \Eb_{x_m}\left[p^{l-1}_m(\pi_m(x_{m})|x_{m}) \left(\Delta^{l}_{m}(\pi_m(x_{m}))\right)^2\mid \Upsilon^{l-1}\right].
    \end{align*}

    Furthermore, conditioned on $\Upsilon^{l-1}$, we can obtain that
    \begin{align*}
        &V_{m}(p^{l-1}_m, \pi_m\mid \Upsilon^{l-1}) \cdot \Eb_{x_{m},a_{m}^{l-1}}\left[\left(\Delta^{l}_{m}(a_{m}^{l-1})\right)^2\mid \Upsilon^{l-1}\right]\\
        & = \Eb_{x_{m}}\left[\frac{1}{p^{l-1}_m(\pi_m(x_{m})|x_{m})}\mid \Upsilon^{l-1}\right]\Eb_{x_{m}, a_{m}^{l-1}}\left[\left(\Delta^{l}_{m}(a_{m}^{l-1})\right)^2\mid \Upsilon^{l-1}\right]\\
        & \geq \left(\Eb_{x_{m}}\left[\sqrt{\frac{1}{p^{l-1}_m(\pi_m(x_{m})|x_{m})}\Eb_{a_{m}^{l-1}}\left[\left(\Delta^{l}_{m}(a_{m}^{l-1})\right)^2\right]}\mid \Upsilon^{l-1}\right]\right)^2\\
        & \geq \left(\Eb_{x_{m}}\left[\sqrt{\frac{1}{p^{l-1}_m(\pi_m(x_{m})|x_{m})}p^{l-1}_m(\pi_m(x_{m})|x_{m})  \left(\Delta^{l}_{m}(\pi_m(x_{m}))\right)^2}\mid \Upsilon^{l-1}\right]\right)^2\\
        & = \left(\Eb_{x_{m}}\left[\left|\Delta^{l}_{m}(\pi_m(x_{m}))\right|\mid \Upsilon^{l-1}\right]\right)^2\\
        & \geq \left|\widehat{\Rc}_{m}^l(\pi_m\mid \Upsilon^{l-1})- \Rc_m(\pi_m)\right|^2.
    \end{align*}
    
    As a result, it holds that
    \begin{align*}
         \left|\widehat{\Rc}_{m}^l(\pi_m\mid \Upsilon^{l-1}) - \Rc_m(\pi_m)\right| &\leq  \sqrt{V_{m}(p^{l-1}_m, \pi_m\mid \Upsilon^{l-1})} \sqrt{\Ec_m^{l-1}(\Upsilon^{l-1})},
    \end{align*}
    where the last step we use the definition that
    \begin{align*}
        \Ec_m^{l-1}(\Upsilon^{l-1}) = \Eb_{x_{m}, a_{m}^{l-1}}\left[\left(\widehat{f}^{l}(x_{m}, a_{m}^{l-1}) - f^*(x_{m}, a_{m}^{l-1}) \right)^2\mid \Upsilon^{l-1}\right].
    \end{align*}
    This concludes the proof.
\end{proof}

Furthermore, the following lemma provides a characterization of the relation between the virtual loss $\widehat{\Reg}^l_m$ and the true loss $\Reg^l_m$.
\begin{lemma}\label{lem:true_virtual_loss}
    For any epochs $l \geq 1$, for any policies $\pi_{[M]}\in \Psi_{[M]}$, it holds that
    \begin{align*}
        \sum_{m\in [M]} E^l_m \Reg_m(\pi_m) \leq 2\sum_{m\in [M]} E^l_m  \Eb_{\Upsilon^{l-1}}\left[\widehat{\Reg}_{m}^l(\pi_m\mid \Upsilon^{l-1}) \right] + \eta^l,\\
        \sum_{m\in [M]} E^l_m \Eb_{\Upsilon^{l-1}}\left[\widehat{\Reg}_{m}^l(\pi_m\mid \Upsilon^{l-1}) \right] \leq 2\sum_{m\in [M]} E^l_m\Reg_{m}(\pi_m) + \eta^l,
    \end{align*}
    with
    \begin{align*}
        \eta^{l} := \frac{9c^2}{\gamma^l} \sum_{m\in [M]}E^l_mK_m.
    \end{align*}
\end{lemma}
\begin{proof}
    First, we note that for $l = 1$, it holds that
    \begin{align*}
        \sum_{m\in [M]} E^1_m \Reg_m(\pi_m)&\leq \sum_{m\in [M]} E^1_m\leq \eta^1 = 9c^2\sum_{m\in [M]}E^1_m K_m;\\
        \sum_{m\in [M]} E^1_m \widehat{\Reg}^l_m(\pi_m) & = 0 \leq \eta^1 = 9c^2\sum_{m\in [M]} E^1_m K_m,
    \end{align*}
    which means the lemma holds for the first epoch.

    We then perform an inductive proof and start by assuming that for epoch $l-1$ and any policies $\pi_{m}\in \Psi_m$, it holds that
    \begin{align*}
        \sum_{m\in [M]} E^{l-1}_m\Reg_m(\pi_m) \leq 2\sum_{m\in [M]} E^{l-1}_m\Eb_{\Upsilon^{l-2}}\left[\widehat{\Reg}_{m}^{l-1}(\pi_m\mid \Upsilon^{l-2})\right] + \eta^{l-1}\\
        \sum_{m\in [M]} E^{l-1}_m\Eb_{\Upsilon^{l-2}}\left[\widehat{\Reg}^{l-1}_m(\pi_m\mid \Upsilon^{l-2})\right] \leq 2\sum_{m\in [M]} E^{l-1}_m\Reg_{m}(\pi_m) + \eta^{l-1}.
    \end{align*}
    
    Then, it can be observed that
    \begin{align*}
        &\Reg_m(\pi_m) - \widehat{\Reg}_{m}^l(\pi_m\mid \Upsilon^{l-1}) \\
        &= \Rc_m(\pi^*_m) - \Rc_m(\pi_m)- \left(\widehat{\Rc}_{m}^l(\widehat{\pi}^{l}_m\mid \Upsilon^{l-1}) - \widehat{\Rc}_{m}^l(\pi_m\mid \Upsilon^{l-1})\right)\\
        & \leq \Rc_m(\pi^*_m) - \Rc_m(\pi_m) - \left(\widehat{\Rc}_{m}^l(\pi^*_m\mid \Upsilon^{l-1}) - \widehat{\Rc}_{m}^l(\pi_m\mid \Upsilon^{l-1})\right)\\
        & =  \Rc_m(\pi^*_m)  - \widehat{\Rc}_{m}^l(\pi^*_m\mid \Upsilon^{l-1})  + \widehat{\Rc}_{m}^l(\pi_m\mid \Upsilon^{l-1}) - \Rc_m(\pi_m)\\
        & \overset{(a)}{\leq} \sqrt{V_{m}(p^{l-1}_m, \pi_m^*\mid \Upsilon^{l-1}) } \sqrt{\Ec_m^{l-1}(\Upsilon^{l-1})} + \sqrt{V_{m}(p^{l-1}_m, \pi_m\mid \Upsilon^{l-1}) } \sqrt{ \Ec_m^{l-1}(\Upsilon^{l-1})}\\
        & \leq \frac{V_{m}(p^{l-1}_m, \pi_m^*\mid \Upsilon^{l-1})}{8c\gamma^l} + \frac{V_{m}(p^{l-1}_m, \pi_m\mid \Upsilon^{l-1})}{8c\gamma^l} + 4c\gamma^l\Ec_m^{l-1}(\Upsilon^{l-1})\\
        & \overset{(b)}{\leq} \frac{K_m + \gamma^{l-1} \widehat{\Reg}^{l-1}_{m}(\pi^*_m\mid \Upsilon^{l-1})}{8c\gamma^l} + \frac{K_m + \gamma^{l-1} \widehat{\Reg}^{l-1}_{m}(\pi_m\mid \Upsilon^{l-1})}{8c\gamma^l} + 4c\gamma^l\Ec_m^{l-1}(\Upsilon^{l-1}),
    \end{align*}
    where inequality (a) is from Lemma~\ref{lem:VE} and inequality (b) is from Lemma~\ref{lem:V_upper}.
    
    Then, summing over all $M$ agents, we can obtain that
    \begin{align*}
        &\Eb_{\Upsilon^{l-1}}\left[\sum_{m\in [M]}E^{l}_m \left(\Reg_m(\pi_m) - \widehat{\Reg}_{m}^l(\pi_m\mid \Upsilon^{l-1})\right)\right]\\
        &\leq \frac{\sum_{m\in [M]}E^{l}_mK_m}{4c\gamma^l}+ \frac{\gamma^{l-1}}{8c\gamma^{l}}\sum_{m\in [M]}E^{l}_m\Eb_{\Upsilon^{l-1}}\left[\widehat{\Reg}^{l-1}_{m}(\pi^*_m\mid \Upsilon^{l-1})\right] \\
        &\quad + \frac{\gamma^{l-1}}{8c\gamma^{l}}\sum_{m\in [M]}E^{l}_m\Eb_{\Upsilon^{l-1}}\left[\widehat{\Reg}^{l-1}_{m}(\pi_m\mid \Upsilon^{l-1})\right] + 4c\gamma^l \sum_{m\in [M]}E^{l}_m\Eb_{\Upsilon^{l-1}}\left[\Ec_m^{l-1}(\Upsilon^{l-1})\right]\\
        &\overset{(d)}{\leq} \frac{\sum_{m\in [M]}E^{l}_mK_m}{4c\gamma^l}+ \frac{\overline{c}\gamma^{l-1}}{8c\gamma^{l}}\sum_{m\in [M]}E^{l-1}_m\Eb_{\Upsilon^{l-1}}\left[\widehat{\Reg}^{l-1}_{m}(\pi^*_m\mid \Upsilon^{l-1})\right] \\
        &\quad + \frac{\overline{c}\gamma^{l-1}}{8c\gamma^{l}}\sum_{m\in [M]}E^{l-1}_m\Eb_{\Upsilon^{l-1}}\left[\widehat{\Reg}^{l-1}_{m}(\pi_m\mid \Upsilon^{l-1})\right] + 4c\gamma^l \sum_{m\in [M]}E^{l}_m\Eb_{\Upsilon^{l-1}}\left[\Ec_m^{l-1}(\Upsilon^{l-1})\right]\\
        & \overset{(e)}{\leq} \frac{\sum_{m\in [M]}E^{l}_mK_m}{4c\gamma^l} + \frac{\overline{c}\gamma^{l-1}}{4c\gamma^{l}}\sum_{m\in [M]}E^{l-1}_m\Reg_{m}(\pi_m)+ \frac{\overline{c}\gamma^{l-1}}{4c\gamma^{l}}\cdot \eta^{l-1} \\
        &\quad + 4c\gamma^l \sum_{m\in [M]}E^{l}_m\Eb_{\Upsilon^{l-1}}\left[\Ec_m^{l-1}(\Upsilon^{l-1})\right]\\
        & \overset{(f)}{\leq} \frac{\sum_{m\in [M]}E^{l}_mK_m }{4c\gamma^l}+ \frac{1}{4}\sum_{m\in [M]}E^{l}_m\Reg_{m}(\pi_m)+ \frac{9c^2\sum_{m\in [M]}E^{l}_mK_m}{4\gamma^{l}} + \frac{4c^2\sum_{m\in [M]}E^{l}_m K_m}{\gamma^l},
    \end{align*}
    where inequality (d) is from the definition $\overline{c} := \max_{m\in [M], l\in [2, l(T)]} E^l_m/E^{l-1}_m$. Inequality (e) is from the induction assumption that
    \begin{align*}
        \sum_{m\in [M]}E^{l-1}_m\Eb_{\Upsilon^{l-1}}\left[\widehat{\Reg}^{l-1}_{m}(\pi^*_m\mid \Upsilon^{l-1})\right] &= \sum_{m\in [M]} E^{l-1}_m\Eb_{\Upsilon^{l-2}}\left[\widehat{\Reg}^{l-1}_m(\pi^*_m\mid \Upsilon^{l-2})\right] \\
        &\leq 2\sum_{m\in [M]} E^{l-1}_m\Reg_{m}(\pi^*_m) + \eta^{l-1} = \eta^{l-1},\\
        \sum_{m\in [M]}E^{l-1}_m\Eb_{\Upsilon^{l-1}}\left[\widehat{\Reg}^{l-1}_{m}(\pi_m\mid \Upsilon^{l-1})\right] &= \sum_{m\in [M]} E^{l-1}_m\Eb_{\Upsilon^{l-2}}\left[\widehat{\Reg}^{l-1}_m(\pi_m\mid \Upsilon^{l-2})\right] \\
        &\leq 2\sum_{m\in [M]} E^{l-1}_m\Reg_{m}(\pi_m) + \eta^{l-1}.
    \end{align*}
    Inequality (f) is based on the definition $\underline{c} := \min_{m\in [M], l\in [2, l(T)]} E^l_m/E^{l-1}_m$, $c := \overline{c}/\underline{c}$ and $\eta^{l} := 9c^2 \sum_{m\in [M]}E^l_mK_m/\gamma^l$, also the assumption that $\gamma^{l}\geq \gamma^{l-1}$ and Lemma~\ref{lem:learning_rate}, which indicates that
    \begin{align*}
        \Eb_{\Upsilon^{l-1}}\left[\sum_{m\in [M]} E^{l-1}_m \Ec_m^{l-1}(\Upsilon^{l-1})\right] \leq \frac{\sum_{m\in [M]}E^{l-1}_m K_m}{(\gamma^l)^2}.
    \end{align*}
    
    Thus, we can obtain that
    \begin{align*}
        \frac{3}{4}\sum_{m\in [M]}E^l_m\Reg_m(\pi_m) &\leq \sum_{m\in [M]}E^l_m \Eb_{\Upsilon^{l-1}}\left[\widehat{\Reg}_{m}^l(\pi_m\mid\Upsilon^{l-1})\right] + \frac{\sum_{m\in [M]}E^l_mK_m}{4c\gamma^l}  \\
        &+\frac{25c^2\sum_{m\in [M]}E^{l}_m K_m}{4\gamma^l}\\
        \Rightarrow \sum_{m\in [M]}E^l_m \Reg_m(\pi_m) &\leq \frac{4}{3}\sum_{m\in [M]}E^l_m \Eb_{\Upsilon^{l-1}}\left[\widehat{\Reg}_{m}^l(\pi_m\mid\Upsilon^{l-1})\right] + \frac{\sum_{m\in [M]}E^l_mK_m}{3c\gamma^l}  \\
        &+ \frac{25c^2\sum_{m\in [M]}E^{l}_m K_m}{4\gamma^l}\\
        & \leq 2\sum_{m\in [M]}E^l_m \Eb_{\Upsilon^{l-1}}\left[\widehat{\Reg}_{m}^l(\pi_m\mid\Upsilon^{l-1})\right] + \eta^l
    \end{align*}

    Also, it similarly holds that
    \begin{align*}
        &\widehat{\Reg}_{m}^l(\pi_m\mid \Upsilon^{l-1}) - \Reg_m(\pi_m)  \\
        &= \widehat{\Rc}_{m}^l(\widehat{\pi}^{l}_m\mid \Upsilon^{l-1}) - \widehat{\Rc}_{m}^l(\pi_m\mid \Upsilon^{l-1}) - \left(\Rc_m(\pi^*_m) - \Rc_m(\pi_m)\right)\\
        & \leq \widehat{\Rc}_{m}^l(\widehat{\pi}^{l}_m\mid \Upsilon^{l-1}) - \widehat{\Rc}_{m}^l(\pi_m\mid \Upsilon^{l-1}) - \left(\Rc_{m}(\widehat{\pi}^l_m) - \Rc_{m}(\pi_m)\right)\\
        & =  \widehat{\Rc}_{m}^l(\widehat{\pi}^{l}_m\mid \Upsilon^{l-1})  - \Rc_{m}(\widehat{\pi}^l_m) + \Rc_{m}(\pi_m) - \widehat{\Rc}_{m}^l(\pi_m\mid \Upsilon^{l-1})\\
        & \leq \sqrt{V_{m}(p^{l-1}_m, \widehat{\pi}_m^l\mid \Upsilon^{l-1}) } \sqrt{\Ec_m^{l-1}(\Upsilon^{l-1})} + \sqrt{V_{m}(p^{l-1}_m, \pi_m\mid \Upsilon^{l-1}) } \sqrt{\Ec_m^{l-1}(\Upsilon^{l-1})}\\
        & \leq \frac{K_m + \gamma^{l-1} \widehat{\Reg}^{l-1}_{m}(\widehat{\pi}_m^l\mid \Upsilon^{l-1})}{8c\gamma^l} + \frac{K_m + \gamma^{l-1} \widehat{\Reg}^{l-1}_{m}(\pi_m \mid \Upsilon^{l-1})}{8c\gamma^l} + 4c\gamma^l\Ec_m^{l-1}(\Upsilon^{l-1}).
    \end{align*}

    Then, summing over $M$ agents, we can obtain that
    \begin{align*}
        &\Eb_{\Upsilon^{l-1}}\left[\sum_{m\in [M]}E^l_m \left( \widehat{\Reg}_{m}^l(\pi_m \mid \Upsilon^{l-1}) - \Reg_m(\pi_m)\right)\right]\\
        & \leq \frac{\sum_{m\in [M]}E^l_mK_m}{4c\gamma^l} + \frac{\overline{c}\gamma^{l-1}}{8c\gamma^{l}}\sum_{m\in [M]}E^{l-1}_m\Eb_{\Upsilon^{l-1}}\left[\widehat{\Reg}^{l-1}_{m}(\widehat{\pi}_m^l\mid \Upsilon^{l-1})\right]\\
        & + \frac{\overline{c}\gamma^{l-1}}{8c\gamma^{l}}\sum_{m\in [M]}E^{l-1}_m\Eb_{\Upsilon^{l-1}}\left[\widehat{\Reg}^{l-1}_{m}(\pi_m\mid \Upsilon^{l-1})\right] +  4c\gamma^l\sum_{m\in [M]}E^l_m \Eb_{\Upsilon^{l-1}}\left[\Ec_m^{l-1}(\Upsilon^{l-1})\right]\\
        &\leq \frac{\sum_{m\in [M]}E^l_mK_m}{4c\gamma^l}+ \frac{\overline{c}\gamma^{l-1}}{4c\gamma^{l}}\sum_{m\in [M]}E^{l-1}_m\Eb_{\Upsilon^{l-1}}\left[\Reg_{m}(\widehat{\pi}^l_m\mid \Upsilon^{l-1}) \right]\\
        & + \frac{\overline{c}\gamma^{l-1}}{4c\gamma^{l}}\sum_{m\in [M]}E^{l-1}_m\Reg_{m}(\pi_m)+ \frac{\overline{c}\gamma^{l-1}}{4c\gamma^{l}}\cdot \eta^{l-1}  + 4c\gamma^l\sum_{m\in [M]}E^l_m \Eb_{\Upsilon^{l-1}}\left[\Ec_m^{l-1}(\Upsilon^{l-1})\right]\\
        &\overset{(g)}{\leq} \frac{\sum_{m\in [M]}E^l_m K_m}{4c\gamma^{l}}+ \frac{\gamma^{l-1}}{4\gamma^{l}} \cdot \eta^l+ \frac{\gamma^{l-1}}{4\gamma^{l}}\sum_{m\in [M]}E^l_m\Reg_{m}(\pi_m) \\
        & + \frac{\overline{c}\gamma^{l-1}}{4c\gamma^{l}}\cdot \eta^{l-1}  + 4c\gamma^l\sum_{m\in [M]}E^l_m \Eb_{\Upsilon^{l-1}}\left[\Ec_m^{l-1}(\Upsilon^{l-1})\right]\\
        &\leq \frac{\sum_{m\in [M]}E^l_m K_m}{4c\gamma^{l}} + \frac{9c^2\sum_{m\in [M]}E^l_mK_m}{4\gamma^l}  + \frac{1}{4}\sum_{m\in [M]}E^l_m\Reg_{m}(\pi_m) \\
        & + \frac{9c^2\sum_{m\in [M]}E^l_mK_m}{4\gamma^l}  + \frac{4c^2\sum_{m\in [M]}E^{l}_m K_m}{\gamma^l},
    \end{align*}
    where inequality (g) is from the previous derivation that
    \begin{align*}
        \sum_{m\in [M]}E^{l-1}_m\Reg_{m}(\widehat{\pi}^l_m\mid \Upsilon^{l-1})  \leq 2\underline{c}\sum_{m\in [M]}E^{l}_m\widehat{\Reg}^l_{m}(\widehat{\pi}^l_m\mid \Upsilon^{l-1}) + \underline{c} \eta^{l} = \underline{c} \eta^{l}
    \end{align*}
    
    Thus, it holds that
    \begin{align*}
        \sum_{m\in [M]}E^l_m \Eb_{\Upsilon^{l-1}}\left[\widehat{\Reg}^{l-1}_{m}(\widehat{\pi}_m^l\mid \Upsilon^{l-1})\right]&\leq \frac{5}{4}\sum_{m\in [M]}E^l_m\Reg_{m}(\pi_m) \\
        &+ \frac{\sum_{m\in [M]}E^l_m K_m}{4c\gamma^{l}} + \frac{17c^2\sum_{m\in [M]}E^l_mK_m}{2\gamma^l}\\
        \Rightarrow \sum_{m\in [M]}E^l_m \Eb_{\Upsilon^{l-1}}\left[\widehat{\Reg}^{l-1}_{m}(\widehat{\pi}_m^l\mid \Upsilon^{l-1})\right]&\leq 2\sum_{m\in [M]}E^l_m\Reg_{m}(\pi_m) + \eta^l.
    \end{align*}

    With these two parts, the lemma can be obtained by induction.
\end{proof}

Furthermore, the following lemma provides a characterization of the per-epoch loss of the federation.
\begin{lemma}\label{lem:epoch_regret_bound}
    For every epoch $l>1$, conditioned on $\Upsilon^{l-1}$, it holds that
    \begin{align*}
        \Eb_{\Upsilon^{l-1}}\left[\sum_{m\in [M]} E^l_m \sum_{\pi_m\in \Psi_m} Q^{l}_m(\pi_m\mid \Upsilon^{l-1})\Reg_m(\pi_m)\right] \leq \frac{11c^2}{\gamma^l} \sum_{m\in [M]}E^l_mK_m,
    \end{align*}
    where $Q^l(\cdot|\Upsilon^{l-1})$ is a probability measure on $\Psi_m$ defined in Lemma~\ref{lem:q_distribution}
\end{lemma}

\begin{proof}
    For any probability measures $\{\tilde{Q}^l_m(\cdot): m\in [M]\}$, where $\tilde{Q}^l_m(\cdot)$ is on $\Psi_{M}$, it holds that
    \begin{align*}
        &\sum_{m\in [M]} E^l_m \sum_{\pi_m\in \Psi_m} \tilde{Q}^{l}_m(\pi_m)\Reg_m(\pi_m)\\
        & \overset{(a)}{\leq} 2\Eb_{\Upsilon^{l-1}}\left[\sum_{\pi_{[M]}\in \Psi_{[M]}} \tilde{Q}^{l}(\pi_{[M]}) \sum_{m\in [M]} E^l_m \widehat{\Reg}_m(\pi_m \mid \Upsilon^{l-1})\right] + \eta^l\\
        & = 2\Eb_{\Upsilon^{l-1}}\left[\sum_{m\in [M]} E^l_m \sum_{\pi_m\in \Psi_m} \tilde{Q}^{l}_m(\pi_m)\widehat{\Reg}_m(\pi_m\mid \Upsilon^{l-1})\right] + \eta^l,
    \end{align*}
    where inequality (a) is from Lemma~\ref{lem:true_virtual_loss} and $\tilde{Q}^l(\pi_{[M]}): = \prod_{m\in [M]}\tilde{Q}^l_m(\pi_m)$.
    Thus, we can obtain that
    \begin{align*}
        &\Eb_{\Upsilon^{l-1}}\left[\sum_{m\in [M]} E^l_m \sum_{\pi_m\in \Psi_m} Q^{l}_m(\pi_m\mid \Upsilon^{l-1})\Reg_m(\pi_m)\right]\\
        & \leq 2\Eb_{\Upsilon^{l-1}}\left[\sum_{m\in [M]} E^l_m \sum_{\pi_m\in \Psi_m} Q^{l}_m(\pi_m\mid \Upsilon^{l-1})\widehat{\Reg}_m(\pi_m\mid \Upsilon^{l-1})\right] + \eta^l\\
        & \overset{(b)}{\leq} \frac{2}{\gamma^l}\sum_{m\in [M]} E^l_m K_m + \frac{9c^2}{\gamma^l} \sum_{m\in [M]}E^l_mK_m\\
        & \leq \frac{11c^2}{\gamma^l} \sum_{m\in [M]}E^l_mK_m,
    \end{align*}
    where inequality (b) is from Lemma~\ref{lem:q_loss_bound}.
\end{proof}

With the previous lemmas, we can obtain the final Theorem~\ref{thm:global_regret}, which is restated in the following. 
\begin{theorem}[Restatement of Theorem~\ref{thm:global_regret}]
    Using a learning rate 
    \begin{align*}
        \gamma^ l = O\left(\sqrt{\sum_{m\in [M]}E^{l-1}_m K_m/\left(\sum_{m\in [M]}E^{l-1}_m\Ec(E^{l-1}_{[M]})\right)}\right)
    \end{align*}
    in epoch $l$, denoting $\bar{K}^l := \sum_{m\in [M]}E^l_m K_m/\sum_{m\in [M]}E^l_m$,
    the regret of FedIGW can be bounded as
        \begin{align*}
            \Reg(T) = O\left(\sum_{m\in [M]}E^1_m  + \sum_{l\in [2,l(T)]}c^\frac{5}{2}\sqrt{\bar{K}^l \Ec(E^{l-1}_{[M]})}\sum_{m\in [M]}E^l_m\right).
        \end{align*}
    \end{theorem}
\begin{proof}[Proof of Theorem~\ref{thm:global_regret}]
    The expected regret can be bounded as
    \begin{align*}
        &\Reg(T) = \Eb\left[\sum_{m\in [M]}\sum_{t_m\in [T_m]} \left(f^*(x_{m,t_m}, \pi^*_m(x_{m,t_m})) - f^*(x_{m,t_m}, a_{m,t_m})\right)\right]\\
        & \leq \Eb\left[\sum_{l\in [2, l(T)]}\sum_{m\in [M]}\sum_{t_m \in [t_m(\tau^{l-1})+1, t_m(\tau^l)]} \left(f^*(x_{m,t_m}, \pi^*_m(x_{m,t_m})) - f^*(x_{m,t_m}, a_{m,t_m})\right)\right] + \sum_{m\in [M]}E^1_m\\
        & = \sum_{l\in [2, l(T)]}\Eb_{\Upsilon^{l-1}}\left[\Eb_{x_m, a^l_m}\left[\sum_{m\in [M]}E^l_m \left(f^*(x_{m}, \pi^*_m(x_{m})) - f^*(x_{m}, a_{m})\right) \mid \Upsilon^{l-1}\right]\mid \Upsilon^{l-1}\right]+\sum_{m\in [M]}E^1_m\\
        & \overset{(a)}{=} \sum_{l\in [2, l(T)]}\Eb_{\Upsilon^{l-1}}\left[\sum_{m\in [M]}E^l_m \sum_{\pi_m\in \Psi^m} Q^l_m(\pi_m\mid \Upsilon^{l-1}) \Reg_m(\pi_m)\mid \Upsilon^{l-1}\right] + \sum_{m\in [M]}E^1_m\\
        & \overset{(b)}{\leq} \sum_{l\in [2, l(T)]}\frac{11c^2}{\gamma^l} \sum_{m\in [M]}E^l_mK_m + \sum_{m\in [M]}E^1_m\\
        & \overset{(c)}{=} \sum_{l\in [2, l(T)]}11c^2 \sqrt{\frac{\sum_{m\in [M]}E^{l-1}_m\Ec(\Fc; E^{l-1}_{[M]})}{\sum_{m\in [M]}E^{l-1}_m K_m}} \sum_{m\in [M]}E^l_mK_m + \sum_{m\in [M]}E^1_m\\
        & \leq \sum_{l\in [2, l(T)]}11c^2 \sqrt{\overline{K}\Ec(\Fc; E^{l-1}_{[M]})} \sum_{m\in [M]}E^{l-1}_m + \sum_{m\in [M]}E^1_m,
    \end{align*}
    where equality (a) is from Lemma~\ref{lem:loss_decompose}, inequality (b) is from Lemma~\ref{lem:epoch_regret_bound}, and inequality (c) is from the choice of $\gamma^l$. The proof is then concluded.
\end{proof}

\subsection{Supporting Lemmas}
The following supporting lemmas can be similarly obtained by the corresponding proofs in \cite{simchi2022bypassing}.
\begin{lemma}[Lemma 3, \cite{simchi2022bypassing}]\label{lem:q_distribution}
    For any epoch $l\in \Nb$, conditioned on $\Upsilon^{l-1}$, there exists a probability measure $Q^l_m(\cdot|\Upsilon^{l-1})$ on $\Psi_m$ such that
    \begin{align*}
        \forall a_m\in \Ac_m, \forall x_m \in \Xc_m, \qquad p^l_m(a_m|x_m, \Upsilon^{l-1}) = \sum_{\pi_m\in \Psi_m} \oneb\{\pi_m(x_m) =a_m\} Q^l_m(\pi_m|\Upsilon^{l-1}).
    \end{align*}
\end{lemma}

\begin{lemma}[Lemma 4, \cite{simchi2022bypassing}]\label{lem:loss_decompose}
    Fix any epoch $l\in \Nb$, we have
    \begin{align*}
        \Eb_{x_{m}\sim \Dc_m^{\Xc_m}, a_{m}^l\sim p_m^{l}(\cdot|x_m)} \left[f^*(x_m, \pi^*_m(x_m)) - f^*(x_m, a_{m}^l)\mid \Upsilon^{l-1}\right] \\
        = \sum_{\pi_m\in \Psi_m} Q^l_m(\pi_m\mid \Upsilon^{l-1}) \Reg_m(\pi_m).
    \end{align*}
\end{lemma}

\begin{lemma}[Lemma 5, \cite{simchi2022bypassing}]\label{lem:q_loss_bound}
    Fix any epoch $l\in \Nb$, conditioned on $\Upsilon^{l-1}$, we have
    \begin{align*}
        \sum_{\pi\in \Psi_m} Q^l_m(\pi_m\mid \Upsilon^{l-1}) \widehat{\Reg}_{m}^l(\pi_m\mid \Upsilon^{l-1}) \leq \frac{K_m}{\gamma^l}.
    \end{align*}
\end{lemma}

\begin{lemma}[Lemma 6, \cite{simchi2022bypassing}]\label{lem:V_upper}
    Fix any epoch $l\in \Nb$, for any policy $\pi_m\in \Psi_m$, we have
    \begin{align*}
        V_m(p^l_m, \pi_m \mid \Upsilon^{l-1}) \leq K_m + \gamma^l \widehat{\Reg}^l_{m}(\pi_m \mid \Upsilon^{l-1}).
    \end{align*}
\end{lemma}

\section{Proofs for Section~\ref{subsec:concrete}}
\subsection{Proofs of Corollary~\ref{col:finite}}
First, with realizability, i.e., Assumption~\ref{asp:realizable}, the following characterization can be obtained.
\begin{lemma}[Lemma 4.2, \cite{agarwal2012contextual}]\label{lem:finite_var}
    Fix a function $f\in \Fc$. Suppose we sample $x_m, r_m$ from the data distribution $\Dc_m$, and an action $a_m$ from an arbitrary distribution such that $r_m$ and $a_m$ are conditionally independent given $x_m$. Define the random variable
    \begin{align*}
        \ell_m(f) := \left(f(x_m, a_m) - r_m(a_m)\right)^2 - \left(f^*(x_m, a_m) - r_m(a_m)\right)^2.
    \end{align*}
    Then, we have
    \begin{align*}
        \Eb_{x_m, r_m, a_m}\left[\ell_m(f)\right] = \Eb_{x_m, a_m}\left[\left(f(x_m, a_m) - f^*(x_m,a_m)\right)^2 \right]
    \end{align*}
    and
    \begin{align*}
        \Vb_{x_m, r_m, a_m}\left[\ell_m(f)\right] \leq 4\Eb_{x_m, r_m, a_m}\left[\ell_m(f)\right],
    \end{align*}
    where $\Vb[\cdot]$ denotes the variance of a random variable.
\end{lemma}

Then, we establish an upper bound for the excess risk bound required in Definition~\ref{def:error_general} via the following lemma
\begin{lemma}\label{lem:finite_error}
        Under the setup of Assumption~\ref{def:error_general}, if the adopted FL protocol provides an exact minimizer for the optimization problem in Eqn.~\eqref{eqn:FL} with quadratic losses, i.e.,
        \begin{align*}
            \widehat{f} = \argmin_{f\in \Fc} \frac{1}{n}\sum_{m \in [M]}\sum_{i\in [n_m]} \left(f(x_{m}^i, a_{m}^i) - y_{m}^i\right)^2,
        \end{align*}
        then, with probability at least $1-\delta$, it holds that
        \begin{align*}
            \sum_{m\in [M]} \frac{n_m}{n} \cdot \Eb_{x_m\sim \Dc^{\Xc_m}_{m}, a_m\sim p_m(\cdot|x_m)}\left[\left(\widehat{f}(x_m, a_m) - f^*(x_m, a_m)\right)^2\right] \leq \frac{25 \log(|\Fc|/\delta)}{n}.
        \end{align*}
        As a result, Definition~\ref{def:error_general} holds with
        \begin{align*}
            \Ec(\delta, n_{[M]}) \leq O\left(\log(|\Fc|n)/n\right).
        \end{align*}
\end{lemma}
\begin{proof}
    For simplicity, we abbreviate the quadratic loss associated with a fixed function $f\in \Fc$ as
    \begin{align*}
        \ell_{m}^i(f) = \ell_m(f(x_{m}^i, a_{m}^i); r_{m}^i):=\left(f(x_{m}^i, a_{m}^i) - r_{m}^i\right)^2, \qquad \forall m\in [M].
    \end{align*}
    Then, with a probability at least $1-\delta$, for a fixed $f\in \Fc$, it holds that
    \begin{align*}
        &\sum_{m\in [M]}\sum_{i_m \in [n_m]}\Eb_{x_{m}^i,r_{m}^i,a_{m}^i}\left[\ell_{m}^i(f) - \ell_{m}^i(f^*)\right]  - \sum_{m\in [M]}\sum_{i \in [n_m]}\left[\ell_{m}^i(f) - \ell_{m}^i(f^*)\right]\\
        &\overset{(a)}{\leq} 2\sqrt{\sum_{m\in [M]}\sum_{i_m \in [n_m]}\Vb_{x_{m}^i,r_{m}^i,a_{m}^i}\left[\ell_{m}^i(f) - \ell_{m}^i(f^*)\right]\log(1/\delta)}+ \frac{4}{3}\log(1/\delta)\\
        &\overset{(b)}{\leq} 4\sqrt{\sum_{m\in [M]}\sum_{i_m \in [n_m]}\Eb_{x_{m}^i,r_{m}^i,a_{m}^i}\left[\ell_{m}^i(f) - \ell_{m}^i(f^*)\right]\log(1/\delta)}+ \frac{4}{3}\log(1/\delta),
    \end{align*}
    where inequality (a) leverages Bernstein's inequality and inequality (b) is based on Lemma~\ref{lem:finite_var}.
    
    With
    \begin{align*}
        X(f) = \sqrt{\sum_{m\in [M]}\sum_{i_m \in [n_m]}\Eb_{x_{m}^i,r_{m}^i,a_{m}^i}\left[\ell_{m}^i(f) - \ell_{m,i}(f^*)\right]}; \\
        Z(f) = \sum_{m\in [M]}\sum_{i \in [n_m]}\left[\ell_{m}^i(f) - \ell_{m,i}(f^*)\right]; \qquad C = \sqrt{\log(1/\delta)},
    \end{align*}
    applying a union bound to the above inequality indicates that with probability $1- |\Fc|\delta$, for all $f\in \Fc$, it holds that
    \begin{align*}
        X(f)^2 - Z(f) \leq 4CX(f) + \frac{4}{3}C^2 \quad  \Rightarrow  \quad (X(f)- 2C)^2 - Z(f) \leq \frac{16}{3}C^2.
    \end{align*}
    Since $\widehat{f}$ satisfies that $Z(\widehat{f}) \leq 0$, we can obtain that
    \begin{align*}
        X(\widehat{f})^2 \leq 25 C^2,
    \end{align*}
    In other words, with probability $1-\delta$, it holds that
    \begin{align*}
        &\sum_{m\in [M]}\sum_{i_m \in [n_m]}\Eb_{x_{m}^i,r_{m}^i,a_{m}^i}\left[\left(\widehat{f}(x_{m}^i, a_{m}^i) - r_{m}^i\right)^2 - \left(f^*(x_{m}^i, a_{m}^i) - r_{m}^i\right)^2\right]\\
        & = \sum_{m\in [M]}n_m\Eb_{x_{m}^i,a_{m}^i}\left[\left(\widehat{f}(x_{m}^i,a_{m}^i)- f^*(x_{m}^i, a_{m}^i) \right)^2\right] \leq  25 \log(|\Fc|/\delta),
    \end{align*}
    where the equality is from the realizability in Assumption~\ref{asp:realizable}. The first half of the lemma is then proved. 
    
    With $\delta = 1/n$, the second half can be obtained as
    \begin{align*}
        \Eb_{S_{[M]}}\left[\sum_{m\in [M]}\frac{n_m}{n} \cdot \Eb_{x_m, a_m}\left[\left(\widehat{f}(x_m, a_m) - f^*(x_m, a_m)\right)^2\right]\right] \leq \frac{25\log(|\Fc|n)}{n} + \frac{1}{n},
    \end{align*}
    which concludes the proof.
\end{proof}

Based on the established excess risk bound, Corollary~\ref{col:finite} can be obtained as follows.
\begin{corollary}[Restatement of Corollary~\ref{col:finite}]
    If $|\Fc|< \infty$ and the adopted FL protocol provides an exact minimizer for Eqn.~\eqref{eqn:FL} with quadratic losses, with $\tau^l = 2^l$, FedIGW incurs a regret of
    \begin{align*}
        \Reg(T) = O(\sqrt{KMT\log(|\Fc|MT)})
    \end{align*}
    and a total $O(\log(T))$ calls of the adopted FL protocol.
\end{corollary}
\begin{proof}[Proof of Corollary~\ref{col:finite}]
    With Theorem~\ref{thm:global_regret} and Lemma~\ref{lem:finite_error}, under the choice of $\tau^l = 2^l$, the regret can be bounded as
    \begin{align*}
        \Reg(T) &= O\left(ME^1  + \sum_{l\in [2,l(T)]}\sqrt{KME^l \log(|\Fc|ME^l)}\right)\\
        & = O\left(\sum_{l\in [2, \lceil\log_2(T)\rceil]}\sqrt{KM 2^l \log(|\Fc|MT)}\right)\\
        & = O\left(\sqrt{KMT\log(|\Fc|MT)}\right),
    \end{align*}
    and the exponentially growing epoch length naturally leads to $O(\log(T))$ calls of the adopted FL protocol, which concludes the proof.
    \end{proof}

    \subsection{Proofs of Corollary~\ref{col:convex_raw} and Additional Results}
    In the following, we first prove Lemma~\ref{lem:error_decompose} while also noting that this result is general and does not rely on the specific parameterization of $\Fc$, although we presented it with the $d$-dimensional parameterization considered in Section~\ref{subsec:concrete}.
\begin{lemma}[Complete Version of Lemma~\ref{lem:error_decompose}]
        If the loss function $l_m(\cdot; \cdot)$ is $\mu_f$-strongly convex in its first coordinate for all $m\in [M]$, i.e.,
        \begin{align*}
            l_m(z_1'; z_2) - l_m(z_1; z_2) \geq \frac{\textup{d} l_m(z_1; z_2)}{\textup{d} z_1} \cdot (z_1' - z_1)  + \frac{\mu_f}{2}(z_1' - z_1)^2,\quad  \text{for any $z_1, z_1'$ and $z_2$,}
        \end{align*}
        and 
        \begin{align}\label{eqn:opt_omega_star} 
            \inf_{y \in \Rb} \Eb_{r_m}[l_m(y, r_m(a_m)) | x_m, a_m] = \Eb_{r_m}[ l(f_{\omega^*}(x_m,a_m), r_m(a_m)) | x_m, a_m]
        \end{align}
        for all $m\in [M]$, $(x_m, a_m) \in \Xc_m \times \Ac_m$,
        Definition~\ref{def:error_general} holds with
        \begin{align*}
            \Ec(\Fc; n_{[M]}) \geq 2\left(\varepsilon_{\text{opt}}(\Fc; n_{[M]}) + \varepsilon_{\text{gen}}(\Fc; n_{[M]})\right)/\mu_f,
        \end{align*}
        where 
        \begin{align*}
            \varepsilon_{\text{gen}}(\Fc; n_{[M]}):= \Eb_{\Sc,\xi}[\Lc(f_{\widehat{\omega}_{\Sc}})- \widehat{\Lc}(f_{\widehat{\omega}_{\Sc}}; \Sc)];\\
            \varepsilon_{\text{opt}}(\Fc; n_{[M]}):= \Eb_{\Sc, \xi}[\widehat{\Lc}(f_{\widehat{\omega}_{\Sc}}; \Sc) - \widehat{\Lc}(f_{\omega^*_{\Sc}}; \Sc)].
        \end{align*}
    \end{lemma}
    \begin{proof}
        First, for any $\widehat{\omega}_{\Sc}$, it holds that
        \begin{align*}
            &\Lc(f_{\widehat{\omega}_{\Sc}}) - \Lc(f_{\omega^*})\\
            & = \sum_{m\in [M]}\frac{n_m}{n}\Eb_{x_{m,i},a_{m,i},r_{m,i}}\left[\ell(f_{\widehat{\omega}_{\Sc}}(x_{m,i}, a_{m,i}); r_{m,i}) - \ell(f_{\omega^*}(x_{m,i}, a_{m,i}); r_{m,i})\right]\\
            & \geq \frac{\mu_f}{2}\sum_{m\in [M]}\frac{n_m}{n}\Eb_{x_{m,i},a_{m,i}}\left[\left(f_{\widehat{\omega}_{\Sc}}(x_{m,i}, a_{m,i}) - f_{\omega^*}(x_{m,i}, a_{m,i})\right)^2\right]
        \end{align*}
        where the inequality is due to the strong convexity of $\ell(\cdot; \cdot)$ w.r.t. its first coordinate and the optimality of $f_{\omega^*}$ assumed in Eqn.~\eqref{eqn:opt_omega_star}. Thus, we obtain that
        \begin{align*}
            \sum_{m\in [M]}\frac{n_m}{n}\Eb_{x_{m,i},a_{m,i}}\left[\left(f_{\widehat{\omega}_{\Sc}}(x_{m,i}, a_{m,i}) - f_{\omega^*}(x_{m,i}, a_{m,i})\right)^2\right] \leq \frac{2}{\mu_f}\left(\Lc(f_{\widehat{\omega}_{\Sc}}) - \Lc(f_{\omega^*})\right).
        \end{align*}
        Furthermore, it holds that
        \begin{align*}
            &\Eb_{\Sc, \xi}\left[\Lc(f_{\widehat{\omega}_{\Sc}})\right] - \Lc(f_{\omega^*}) \\&= \Eb_{\Sc, \xi}\left[\Lc(f_{\widehat{\omega}_{\Sc}})\right] - \Eb_{\Sc, \xi}\left[\widehat{\Lc}(f_{\widehat{\omega}_{\Sc}}; \Sc)\right] + \Eb_{\Sc, \xi}\left[\widehat{\Lc}(f_{\widehat{\omega}_{\Sc}}; \Sc)\right] - \Lc(f_{\omega^*})\\
            & \leq  \Eb_{\Sc, \xi}\left[\Lc(f_{\widehat{\omega}_{\Sc}})\right] - \Eb_{\Sc, \xi}\left[\widehat{\Lc}(f_{\widehat{\omega}_{\Sc}}; \Sc)\right] + \Eb_{\Sc, \xi}\left[\widehat{\Lc}(f_{\widehat{\omega}_{\Sc}}; \Sc)\right] - \Eb_{\Sc, \xi}\left[\widehat{\Lc}(f_{\omega^*_{\Sc}}; \Sc)\right],
        \end{align*}
        where the last inequality is due to 
        \begin{align*}
            \Lc(f_{\omega^*}) = \Eb_{\Sc}\left[\widehat{\Lc}(f_{\omega^*}; \Sc)\right] \geq \Eb_{\Sc}\left[\widehat{\Lc}(f_{\omega^*_{\Sc}}; \Sc)\right].
        \end{align*}
        The proof is then concluded.
        \end{proof}

        Then, for the generalization error analyses, the following lemma can be obtained via standard proofs (e.g., Theorem 6.4 in \cite{zhang_2023}; Theorem 3.3 in \cite{mohri2018foundations}).
        \begin{lemma}\label{lem:rademacher}
            It holds that
            \begin{align*}
                \varepsilon_{\text{gen}}(\Fc; n_{[M]}): = \Eb_{\Sc,\xi}[\Lc(f_{\widehat{\omega}_{\Sc}})- \widehat{\Lc}(f_{\widehat{\omega}_{\Sc}}; \Sc)] \leq 2\Rf(\Fc; n_{[M]}).
            \end{align*}
            Here, the distributional-independent upper bound $\Rf(\Fc; n_{[M]})$ on the Rademacher complexity is defined as
            \begin{align}\label{eqn:rademacher}
                \Rf(\Fc; n_{[M]}) := \sup\left\{\Eb_{\Sc_{[M]}, \boldsymbol{\sigma}}\left[\sup_{\omega}\left\{\sum_{m\in [M]} \frac{1}{n} \sum_{i \in [n_m]}\sigma_{m,i}\cdot \ell_m(f_{\omega}(x_{m,i}, a_{m,i}); r_{m,i})\right\}\right]\right\},
            \end{align}
            where the outside supremum is over possible distributions of dataset $\Sc$ defined in Definition~\ref{def:error_general} and the expectation is w.r.t. the generation of dataset $\Sc_{[M]}$ following a fixed distribution and independent Rademacher random variables $\boldsymbol{\sigma} := \{\sigma_{m,i}: m\in [M], i\in [n_m]\}$. 
        \end{lemma}

        The optimization error of FedAvg \citep{mcmahan2017communication} and SCAFFOLD \citep{karimireddy2020scaffold} are presented in Appendices~\ref{subapp:FedAvg} and \ref{subapp:SCAFFOLD}. Combining the generalization error and optimization error via Lemma~\ref{lem:error_decompose} into Theorem~\ref{thm:global_regret}, Corollary~\ref{col:convex_raw} can be obtained, which is restated in the following.
        \begin{corollary}[Restatement of Corollary~\ref{col:convex_raw}]\label{col:convex_raw_full}
        Under the condition of Lemma~\ref{lem:error_decompose}, the regret of FedIGW can be bounded as
        \begin{align*}
            \Reg(T) = O\left(ME^1 + \sum_{l\in [2,l(T)]}\sqrt{K\left(\Rf^{l-1} + \varepsilon_{\text{opt}}^l)\right)/\mu_f}ME^l\right),
        \end{align*}
        where $\Rf^{l}: =\Rf(\Fc; \{E^l: m \in [M]\})$ and using $\rho^l$ rounds of agents-server communications (i.e., global aggregations) and $\kappa^l$ rounds of local updates in epoch $l$, under certain assumptions,
        \begin{itemize}[leftmargin=*]
            \item with \textbf{FedAvg} as $\texttt{FLroutine}(\cdot)$, if $\widehat{\Lc}_m(f_{\omega};\Sc^l_{[M]})$ is $\mu_{\omega}$-strongly convex and $\beta_{\omega}$-smooth w.r.t. $\omega$ for all $m\in [M]$ while the gradients are unbiased, have a $\sigma^2_b$-bounded variance and have a $G_b$-bounded dissimilarity, the output $f_{\widehat{\omega}^l}$ satisfies that $
            \varepsilon_{\text{opt}}^l:=\varepsilon_{opt}(\Fc; n^l_{[M]}) \leq \tilde{O}(\sigma_b^2(\mu_{\omega} \rho^l \kappa^l M)^{-1} + \beta_{\omega} G_b^2(\mu_{\omega}\rho^l)^{-2})$, when $\rho^l \geq \Omega(\beta_{\omega}/\mu_{\omega})$ (see Lemma~\ref{lem:FedAvg_complete} for the full statement);
            \item with \textbf{SCAFFOLD} as $\texttt{FLroutine}(\cdot)$, if $\widehat{\Lc}_m(f_{\omega};\Sc^l_{[M]})$ is $\mu_{\omega}$-strongly convex and $\beta_{\omega}$-smooth w.r.t. $\omega$ for all $m\in [M]$ while the gradients are unbiased and have a $\sigma^2_b$-bounded variance, the output $f_{\widehat{\omega}^l}$ satisfies that $
            \varepsilon_{\text{opt}}^l:=\varepsilon_{opt}(\Fc; n^l_{[M]}) \leq \tilde{O}(\sigma_b^2(\mu_{\omega} \rho^l \kappa^l M)^{-1})$, when $\rho^l \geq \Omega(\beta_{\omega}/\mu_{\omega})$ (see Lemma~\ref{lem:SCAFFOLD_complete} for the full statement);.
        \end{itemize}
    \end{corollary}

    By further setting a suitable number of global aggregations for each epoch such that the optimization error is on the same order as the generalization error, the following more specific corollary can obtained for FedAvg and SCAFFOLD, which can be easily extended for other FL designs.
    \begin{corollary}\label{col:convex}
       Under the conditions of Lemma~\ref{lem:error_decompose} and Corollary~\ref{col:convex_raw_full}, FedIGW incurs a regret of
       \begin{align*}
           \Reg(T) = O\left(ME^1 + \sum_{l\in [2,l(T)]}\sqrt{K\Rf^{l-1}/\mu_f}ME^l\right)
       \end{align*}
       with the following bounds on the rounds of communications
       \begin{align*}
           \tilde{O}\left(\sum_{l\in [l(T)]}\frac{\beta_{\omega}}{\mu_{\omega}} + \frac{\sigma_b^2}{\mu_{\omega}\Rf^l \kappa^l M} +  \sqrt{\frac{\beta_{\omega} G_b^2}{\mu_{\omega}^2\Rf^l}}\right) \qquad \text{(using FedAvg)};\\
           \tilde{O}\left(\sum_{l\in [l(T)]}\frac{\beta_{\omega}}{\mu_{\omega}} + \frac{\sigma_b^2}{\mu_{\omega}\Rf^l \kappa^l M} \right) \qquad \text{(using SCAFFOLD)},
       \end{align*}
       where $\Rf^{l}: = \Rf(\Fc_{[M]}, \{E^l: m\in [M]\})$ and $\kappa^l$ is the number of local updates in epoch $l$.
    \end{corollary}
    \begin{proof}
        From Corollary~\ref{col:convex_raw_full}, when using FedAvg as the adopted FL protocol in FedIGW, the optimization error in epoch $l$ of form
        \begin{align*}
            \tilde{O}\left(\frac{\sigma_b^2}{\mu_{\omega}\rho^l \kappa^l M} + \frac{\beta_{\omega} G_b^2}{\mu_{\omega}^2(\rho^l)^2}\right),
        \end{align*}
        when $\rho^l = \Omega(\beta_{\omega}/\mu_{\omega})$. Thus, if the communication rounds
        \begin{align*}
            \rho^l = \tilde{\Theta}\left(\frac{\beta_{\omega}}{\mu_{\omega}} + \frac{\sigma_b^2}{\mu_{\omega}\Rf^l \kappa^l M} + \sqrt{\frac{\beta_{\omega} G_b^2}{\mu_{\omega}^2\Rf^l}}\right).
        \end{align*}
        we are guaranteed to have the optimization error on the order of $O(\Rf^{l})$. 
        
        Then, the regret in Corollary~\ref{col:convex_raw} is of order
        \begin{align*}
           \Reg(T) = O\left(ME^1 + \sum_{l\in [2,l(T)]}\sqrt{K\Rf^{l-1}/\mu_f}ME^l\right)
       \end{align*}
       while the overall communication rounds can be bounded as
       \begin{align*}
           \sum_{l\in [l(T)]} \rho^l = \tilde{O}\left(\sum_{l\in [l(T)]}\frac{\beta_{\omega}}{\mu_{\omega}} + \frac{\sigma_b^2}{\mu_{\omega}\Rf^l \kappa^l M} + \sqrt{\frac{\beta_{\omega} G_b^2}{\mu_{\omega}^2\Rf^l}}\right),
       \end{align*}
       which concludes the proof for FedAvg. The result of using SCAFFOLD can be similarly obtained.
    \end{proof}

    \subsection{A Linear Reward Function Class}\label{subapp:linear}
    We here provide a detailed discussion on the linear reward function class considered in Remark~\ref{rmk:linear} at the end of Section~\ref{subsec:concrete}. Especially, following standard assumptions in linear bandits \citep{abbasi2011improved} and federated linear bandits \citep{li2022asynchronous,he2022simple,amani2022distributed}, we consider $\mu_m(x_m, a_m) = \langle \phi(x_m, a_m), \omega^* \rangle$, where $\phi(\cdot)$ is a known $d$-dimensional mapping and $\omega^*$ is an unknown $d$-dimensional system parameter. 
    Then, it is sufficient to consider a linear function class $\Fc$, where $f_{\omega}(\cdot) = \langle \omega,  \phi(\cdot)\rangle$ and $f^*(\cdot) = \langle \omega^*,  \phi(\cdot)\rangle$. Moreover, for convenience, we assume that $\|\phi(x_m, a_m)\|_2\leq 1$ and $\|\omega^*\|_2 \leq 1$. 
    
    As mentioned in Remark~\ref{rmk:linear}, the FL problem can be formulated as a standard ridge regression with
    \begin{align*}
        \ell_m(f_\omega(x_m, a_m); r_m): = \left(\langle \omega, \phi(x_m, a_m)\rangle - r_m\right)^2 + \lambda \|\omega\|_2^2.
    \end{align*}
    In other words, Eqn.~\eqref{eqn:FL} can be restated as
    \begin{align}\label{eqn:ridge_regression}
        \min_{\omega\in \Rb^d} \widehat{\Lc}(f_{\omega}; \Sc) := \sum_{m\in [M]}\frac{1}{n}\sum_{i\in [n_m]}\left(\langle \omega, \phi(x_m^i, a_m^i)\rangle - r_m^i\right)^2 + \lambda \|\omega\|_2^2,
    \end{align}
    which has an exact minimizer as
    \begin{align}\label{eqn:minimizer_ridge}
        \omega^*_{\Sc} = \left(\frac{1}{n}\sum_{m\in [M]}\sum_{i\in [n_m]} \phi(x_m^i, a_m^i)\phi(x_m^i, a_m^i)^\top + \lambda I\right)^{-1} \left(\frac{1}{n}\sum_{m\in [M]}\sum_{i\in [n_m]}\phi(x_m^i, a_m^i) r_m^i\right).
    \end{align}
    
    We provide an excess risk bound required in Definition~\ref{def:error_general} through the following decomposition: 
    \begin{align*}
        &\Eb_{\Sc,\xi}\left[\sum_{m\in [M]}\frac{n_m}{n}\Eb_{x_m, a_m}\left(\langle \widehat{\omega}_{\Sc}, \phi(x_m, a_m)\rangle - \langle \omega^*, \phi(x_m, a_m)\rangle\right)^2\right]\\
        &\leq 2\Eb_{\Sc,\xi}\left[ \sum_{m\in [M]}\frac{n_m}{n}\Eb_{x_m, a_m}\left(\langle \widehat{\omega}_{\Sc}, \phi(x_m, a_m)\rangle - \langle \omega^*_{\Sc}, \phi(x_m, a_m)\rangle\right)^2 \right] \\
        &\quad + 2\Eb_{\Sc,\xi}\left[ \sum_{m\in [M]}\frac{n_m}{n}\Eb_{x_m, a_m}\left(\langle \omega^*_{\Sc}, \phi(x_m, a_m)\rangle - \langle \omega^*, \phi(x_m, a_m)\rangle \right)^2 \right]\\
        & = 2\Eb_{\Sc,\xi}\left[ \left\|\widehat{\omega}_{\Sc}- \omega^*_{\Sc}\right\|_{\Sigma}^2 \right] \\
        &\quad + 2\Eb_{\Sc}\left[ \sum_{m\in [M]}\frac{n_m}{n}\Eb_{x_m, a_m}\left(\langle \omega^*_{\Sc}, \phi(x_m, a_m)\rangle - \langle \omega^*, \phi(x_m, a_m)\rangle \right)^2 \right]\\
        & \leq 2\Eb_{\Sc,\xi}\left[ \lambda_{\max}(\Sigma) \left\|\widehat{\omega}_{\Sc}- \omega^*_{\Sc}\right\|_2^2 \right] & =: \text{term (A)} \\
        &\quad + 2\Eb_{\Sc}\left[ \sum_{m\in [M]}\frac{n_m}{n}\Eb_{x_m, a_m}\left(\langle \omega^*_{\Sc}, \phi(x_m, a_m)\rangle - \langle \omega^*, \phi(x_m, a_m)\rangle \right)^2 \right] & =: \text{term (B)}
    \end{align*}
    where
    \begin{align*}
        \Sigma := \sum_{m\in [M]}\frac{n_m}{n}\Eb_{x_m, a_m}\left[\phi(x_m, a_m) \phi(x_m, a_m)^\top\right]
    \end{align*}
    and $\lambda_{\max}(\Sigma)$ denotes the maximum eigenvalue of $\Sigma$. With $\|\phi(x,a)\|_2 \leq 1$, it can be verified that $\lambda_{\max}(\Sigma)\leq 1$. In the above decomposition, term (A) can be interpreted as the optimization error, while term (B) is the generalization error. 
    
    We can then plug in the aforementioned explicit formula of $\omega^*_{\Sc}$ into term (B) and demonstrate that $\text{term (B)} = \tilde{O}(d/n)$ with $\lambda = 1/n$ under the assumption that $\|\omega^*\|_2 \leq 1$ and $r_m\in [0,1]$ (e.g., following Theorem 9.35 in \cite{zhang_2023}). 

    For the ridge regression problem in Eqn.~\eqref{eqn:ridge_regression}, previous designs on federated linear bandits typically \citep{wang2019distributed, dubey2020differentially,li2022asynchronous,he2022simple,amani2022distributed} have agents collaboratively provide the exact minimizer in Eqn.~\eqref{eqn:minimizer_ridge} via directly communicating their local rewards aggregates, i.e., $\sum_{i\in [n_m]}\phi(x_m^i, a_m^i) r_m^i$, and local covariance matrices, i.e., $\sum_{i\in [n_m]} \phi(x_m^i, a_m^i)\phi(x_m^i, a_m^i)^\top$. Thus, one round of agent-server communication is sufficient, where $O(Md^2)$ real numbers are shared. However, directly sharing such compressed data is often undesired in FL studies due to privacy concerns. We refer to this protocol as the ``\textbf{direct method}'' for simplicity in the following discussions.
    
    With the flexible FL choice in FedIGW, it can accommodate many other efficient optimization algorithms. In particular, a distributed version of accelerated gradient descent (AGD) \citep{nesterov2003introductory} takes only $O(\sqrt{\kappa}\log(1/\varepsilon'))$ rounds of communications of gradients to have an optimization error of $\varepsilon'$, where $\kappa$ is the condition number (i.e., the ratio between the smooth and strongly convex parameter in the considered problem). With $\lambda = 1/n$, it holds that $\kappa = O(n)$; thus $O(\sqrt{n}\log(d/n))$ rounds of communications of gradients are sufficient to obtain an optimization error of order $\tilde{O}(d/n)$, where each agents' gradients are intuitively $d$-dimensional.

    With the above illustration, the following corollary regarding the performance FedIGW with a linear reward function class is then a straightforward extension from Theorem~\ref{thm:global_regret}.
    \begin{corollary}\label{col:linear}
        In the considered linear reward function class with shared true parameters, using the direct method or distributed AGD as the adopted FL protocol to solve the FL problem in Eqn.~\eqref{eqn:ridge_regression} and $\tau^l = 2^l$, FedIGW obtains a regret of 
        \begin{align*}
            \Reg(T)  = \tilde{O}\left(\sum_{l\in [\log_2(T)]}\sqrt{\frac{Kd}{M2^{l-1}}}M2^l\right) = \tilde{O}\left(\sqrt{MKdT}\right)
        \end{align*}
        and the amount of real numbers communicated can be bounded as
        \begin{align*}
            O\left(\sum_{l\in [\log_2(T)]}Md^2\right) &= O(Md^2\log(T)) \qquad &\text{(using the direct method)};\\
            O\left(\sum_{l\in [\log_2(T)]}Md\sqrt{M2^{l}}\log(d/(M2^{l}))\right) &= O(d\log(d)\sqrt{M^3T}) \qquad &\text{(using distributed AGD)}.
        \end{align*}
    \end{corollary}

    \section{Details of Section~\ref{sec:appendage}}
    \subsection{Personalized Learning: Details of Section~\ref{subsec:per}}\label{app:per}
    Additional details for the personalized learning setting in Section~\ref{subsec:per} are discussed here. In particular, the overall algorithm structure still follows Algorithm~\ref{alg:FedIGW_agent}, while the major difference is that a personalized FL problem is considered:
    \begin{align*}
        \min_{\omega^\alpha, \omega^\beta_{[M]}} \widehat{\Lc}(f_{\omega^\alpha, \omega^\beta_{[M]}}; \Sc_{[M]}) :=\sum_{m\in [M]} \frac{n_m}{n} \widehat{\Lc}_m(f_{\omega^\alpha, \omega^\beta_m}; \Sc_m),
    \end{align*}
    where
    \begin{align*}
        \widehat{\Lc}_m(f_{\omega^\alpha, \omega^\beta_m}; \Sc_m) := \frac{1}{n_m} \sum_{i\in [n_m]}\ell_m(f_{\omega^\alpha, \omega^\beta_m}(x_m^i, a_m^i); r_m^i).
    \end{align*}

    Furthermore, to bound the generalization error, similar to the Rademacher complexity in Eqn.~\eqref{eqn:rademacher}, a slightly different Rademacher complexity is introduced as
    \begin{align*}
        \Pf(\Fc_{[M]}; n_{[M]})=  \sup\left\{\Eb_{\Sc, \boldsymbol{\sigma}}\left[\sup_{\omega^\alpha, \omega^\beta_{[M]}}\left\{\sum_{m\in [M]} \frac{1}{n} \sum_{i \in [n_m]}\sigma_{m,i}\cdot \ell_m(f_{\omega_m}(x_m^i, a_m^i); r_m^i)\right\}\right]\right\},
    \end{align*}
    which is suitable for the considered personalized setting with parameters $[\omega^\alpha, \omega^\beta_{[M]}]$ involved. A similar notation is also adopted in \cite{mohri2019agnostic}.  
    
    The following corollary can then be established for the personalized version of FedIGW with the LSGD-PFL algorithm \citep{hanzely2021personalized} adopted to solve the personalized FL task.

    \begin{corollary}\label{col:convex_per}
       Under the conditions of Lemmas~\ref{lem:error_decompose} and \ref{lem:LSGD-PFL_full}, with LSGD-PFL as the adopted personalized FL protocol, FedIGW incurs a regret of
       \begin{align*}
           \Reg(T) = O\left(ME^1+ \sum_{l\in [2,l(T)]}\sqrt{K\Pf^{l-1}/\mu_f}ME^l\right)
       \end{align*}
       with 
       \begin{align*}
           \tilde{O}\left(\sum_{l\in [l(T)]} \max\{\beta_{\omega^\beta }(\kappa^l)^{-1}, \beta_{\omega^\alpha}\}\mu^{-1}_{\omega} + \sigma^2_b(\mu_{\omega} \kappa^{l} M \Pf^{l})^{-1} + \sqrt{\beta_{\omega^\alpha}(G^2+ \sigma^2)(\mu_{\omega}^2\Pf^{l})^{-1}}\right)
       \end{align*}rounds of communications, where $\Pf^{l}: = \Pf(\Fc_{[M]}, \{E^l: m\in [M]\})$ and $\kappa^l$ is the number of local updates in epoch $l$.
    \end{corollary}

    The proof largely follows that of Corollary~\ref{col:convex}:  decomposing excess risk to generalization and optimization errors; using Rademacher complexity to characterize the generalization error; using FL convergence analyses to characterize the optimization error; and combining them together such that the optimization error does not dominate the generalization error. As the LSGD-PFL protocol \citep{hanzely2021personalized} is adopted to solve the personalized FL task as an illustration, its corresponding convergence analyses should be incorporated, which is presented in Lemma~\ref{lem:LSGD-PFL_full}.

    \subsubsection{A Linear Reward Function Class}\label{subapp:per_linear}
    As an extension of the linear reward function in Appendix~\ref{subapp:linear}, we consider that 
    \begin{align*}
        \mu_m(x_m, a_m) = \langle \phi(x_m, a_m), \omega^*_m \rangle, \qquad \forall m\in [M], (x_m, a_m) \in \Xc_m \times \Ac_m,
    \end{align*}
    and the true model parameters $\{\omega^*_m: m\in [M]\}$ follow Assumption~\ref{asp:per}, i.e., $\omega^*_m = [\omega^{\alpha,*}, \omega^{*,\beta}_m]$ with $\omega^{\alpha,*}$ shared among all agents.

    It can be further realized that the above problem setting is identical to a $\tilde{d}$-dimensional linear system, where $\tilde{d}:= d^{\alpha}+ \sum_{m\in [M]}d^\beta_m$: the overall true model parameter is 
    \begin{align*}
        \tilde{\omega}^* = \left[\omega^{*,\alpha}, \omega^{*,\beta}_1, \cdots,  \omega^{*,\beta}_M\right] \in \Rb^{\tilde{d}}.
    \end{align*}
    and a correspondingly feature mapping $\tilde\phi(\cdot)$ is
    \begin{align*}
        \tilde{\phi}(x_m, a_m) = \left[\phi(x_m, a_m)_{[1:d^{\alpha}]}, \boldsymbol{O}_{d^{\beta}_1}, \cdots, \boldsymbol{O}_{d^{\beta}_{m-1}}, \phi(x_m, a_m)_{[d^{\alpha}+1: d_m]}, \boldsymbol{O}_{d^{\beta}_{m+1}}, \cdots, \boldsymbol{O}_{d^{\beta}_M}\right],
    \end{align*}
    i.e., an expanded version of the original feature, where $\phi(x_m, a_m)_{[i:j]} \in \Rb^{j-i+1}$ denotes the sub-vector containing $[i:j]$-th elements in $\phi(x_m, a_m)$ and $\boldsymbol{O}_{i}\in \Rb^i$ an $i$-dimensional null vector.

    With this reformulated problem, discussions from Appendix~\ref{subapp:linear} can be directly leveraged. Especially, Corollary~\ref{col:linear} indicates the following result.
    \begin{corollary}\label{col:linear_per}
        In the considered linear reward function class with partially true parameters, using distributed AGD as the adopted FL protocol to solve the FL problem in Eqn.~\eqref{eqn:ridge_regression} with reformulated feature mapping $\tilde{\phi}(\cdot)$ and $\tau^l = 2^l$, FedIGW incurs a regret of 
        \begin{align*}
            \Reg(T)  = \tilde{O}\left(\sqrt{MK\tilde{d}T}\right)
        \end{align*}
        and the amount of real numbers communicated can be bounded as $O(d^{\alpha}\log(d^\alpha)\sqrt{M^3T})$.
    \end{corollary}
    
    \subsection{Robustness, Privacy, and Beyond: Details of Section~\ref{subsec:beyond}}\label{app:beyond}
    We here provide some additional discussions on incorporating appendages in FL studies to provide robustness and privacy guarantees for FedIGW among some other directions, e.g., fairness guarantees \citep{mohri2019agnostic, du2021fairness}, client selections \citep{balakrishnan2022diverse,fraboni2021clustered}, and practical communication designs \citep{chen2021distributed, wei2022federated, zheng2020design}. The key is that as long as one FL protocol can provide an estimated function $\widehat{f}$ (which is used in IGW interactions), it can be adopted in FedIGW; thus the desirable properties of the selected FL protocol are naturally inherited to FedIGW. 
    
    For example, \cite{yin2018byzantine, pillutla2022robust, fu2019attack, li2021ditto, zhu2023byzantine} studied how to handle malicious agents, who can deviate arbitrarily from the FL protocol and tamper with their updates, during learning. The commonly adopted approach is to invoke certain robust estimators (e.g., median and trimmed mean). Under suitable assumptions, existing approaches have shown that as long as the proportion of malicious agents does not exceed a threshold (typically, $1/2$), the estimators calculated by federation can still converge within certain amounts of error due to the malicious agents. A recent work \citep{zhu2023byzantine} provides a summary of convergence rates with different robust estimators, which can be leveraged to establish theoretical understandings of FedIGW with robustness.

    On the privacy side, many mechanisms have also been studied in FL \citep{wei2020federated, yin2021comprehensive, liu2022privacy}, to guarantee differential privacy (DP), where the most common approach is to insert noises of suitable scales. Convergence rates have also been established under suitable assumptions, e.g., in \citet{wei2020federated, girgis2021shuffled, wei2021user}. With those analyses, the theoretical behavior of FedIGW with DP can also be similarly established as Corollaries~\ref{col:convex} and \ref{col:convex_per}.

    \section{Algorithm Sketches and Convergence Analyses of FL Designs}\label{app:FL}
    \subsection{FedAvg}\label{subapp:FedAvg}
    The FedAvg algorithm \citep{mcmahan2017communication} is one of the most standard and well-adopted FL protocol. Following it, agents perform local stochastic gradient descents (SGD) with their local objective functions for certain steps and then communicate the updated local models to the server; the server aggregates local models to a global one via a weighted average, which is then communicated to the agents to perform further local SGDs.

    Many theoretical analyses have been provided for FedAvg (e.g., \citet{li2020convergence}). We adopt the one from \citet{karimireddy2020scaffold}  in the following.
    \begin{lemma}[Theorem V in \citet{karimireddy2020scaffold} without client sampling]\label{lem:FedAvg_complete}
        For any dataset $\Sc$, if 
        \begin{itemize}
            \item $\widehat{\Lc}_m(f_{\omega};\Sc_m)$ is $\mu_{\omega}$-strongly convex w.r.t. $\omega$ (see Definition~\ref{def:strongly_convex}) for all $m\in [M]$; 
            \item $\widehat{\Lc}_m(f_{\omega};\Sc_m)$ is $\beta_{\omega}$-smooth w.r.t. $\omega$ (see Definition~\ref{def:smooth}) for all $m\in [M]$;
            \item the stochastic gradients are unbiased and have a $\sigma^2_b$-bounded variance (see Definition~\ref{def:bounded_variance});
            \item the gradients have $G_b$-bounded dissimilarity (see Definition~\ref{def:bounded_dissimilarity}),
        \end{itemize}
        with FedAvg as the adopted FL protocol, the output $\widehat{\omega}$ satisfies that
        \begin{align*}
            \Eb_{\xi}[\widehat{\Lc}(f_{\widehat{\omega}_{\Sc}}; \Sc) - \widehat{\Lc}(f_{\omega^*_{\Sc}}; \Sc) \mid \Sc] \leq   \tilde{O}\left(\frac{\sigma_b^2}{\mu_{\omega}\rho \kappa M} + \frac{\beta_{\omega} G_b^2}{\mu_{\omega}^2\rho^2} + \mu_{\omega} \|\omega^0 - \omega^*_{\Sc}\|_2^2 \exp\left(- \frac{\mu_{\omega}\rho}{16\beta_{\omega}}\right) \right)
        \end{align*}
        when $\rho \geq \frac{8\beta_{\omega}}{\mu_{\omega}}$, where $\rho$ denotes the round of communications (i.e., number of global aggregations), $\kappa$ is the number of local updates (i.e., SGD) between each communication, and $\omega^0$ is the initialization. Note that the last term which decays exponentially w.r.t. $\rho$ is omitted in Corollary~\ref{col:convex_raw_full} and the following derivations for simplicity.
    \end{lemma}

    A few definitions used above are made precise in the following, which are inherited from \cite{karimireddy2020scaffold} and presented here for completeness:
    \begin{definition}[Strongly Convex]\label{def:strongly_convex}
        $\widehat{\Lc}_m(f_{\omega};\Sc)$ is $\mu_{\omega}$-strongly convex w.r.t. $\omega$ for $\mu_{\omega}> 0$ if
        \begin{align*}
            \widehat{\Lc}_m(f_{\omega'};\Sc) - \widehat{\Lc}_m(f_{\omega};\Sc) \geq \left\langle \nabla_{\omega} \widehat{\Lc}_m(f_{\omega};\Sc), \omega' - \omega \right\rangle  + \frac{\mu_{\omega}}{2}\left\|\omega' - \omega\right\|_2^2, \quad \text{for any $\omega$ and $\omega'$.}
        \end{align*}
    \end{definition}
    \begin{definition}[Smooth]\label{def:smooth}
        $\widehat{\Lc}_m(f_{\omega};\Sc)$ is $\beta_{\omega}$-smooth w.r.t. $\omega$ for $\beta_{\omega}> 0$ if
        \begin{align*}
            \widehat{\Lc}_m(f_{\omega'};\Sc) - \widehat{\Lc}_m(f_{\omega};\Sc) \leq \left\langle \nabla_{\omega} \widehat{\Lc}_m(f_{\omega};\Sc), \omega' - \omega \right\rangle  + \frac{\beta_{\omega}}{2}\left\|\omega' - \omega\right\|_2^2, \quad \text{for any $\omega$ and $\omega'$.}
        \end{align*}
    \end{definition}
    \begin{definition}[Stochastic Gradients with Bounded Variances] \label{def:bounded_variance}
    The stochastic gradients have a $\sigma^2_b$-bounded variance if
        \begin{align*}
            \frac{1}{n_m}\sum_{i\in [n_m]}\left\|\nabla_{\omega} \ell_m(f_{\omega}(x_m^i, a_m^i); r_{m}^i) - \nabla_{\omega} \widehat{\Lc}_m(f_{\omega}; \Sc_m)\right\|_2^2 \leq \sigma_b^2, \quad \text{for any $\omega$ and $m$.}
        \end{align*}
    \end{definition}
    \begin{definition}[Gradients with Bounded Dissimilarity]\label{def:bounded_dissimilarity}
        The gradients have a $G_b$-bounded dissimilarity if
        \begin{align*}
            \frac{1}{M}\sum_{m \in [M]}\left\| \nabla_{\omega} \widehat{\Lc}_m(f_{\omega}; \Sc_m)\right\|_2^2 \leq G_b^2, \quad \text{for any $\omega$.}
        \end{align*}
    \end{definition}

    \subsection{SCAFFOLD}\label{subapp:SCAFFOLD}
    The SCAFFOLD algorithm is proposed in \citet{karimireddy2020scaffold}, which enhances FedAvg via leveraging variance reduction to correct drifts in heterogenous agents' local updates. The following result is established in \citet{karimireddy2020scaffold} to characterize the convergence of the SCAFFOLD protocol.
    \begin{lemma}[Theorem VII in \citet{karimireddy2020scaffold} without client sampling]\label{lem:SCAFFOLD_complete}
        For any dataset $\Sc$, if 
        \begin{itemize}
            \item $\widehat{\Lc}_m(f_{\omega};\Sc_m)$ is $\mu_{\omega}$-strongly convex w.r.t. $\omega$ (see Definition~\ref{def:strongly_convex}) for all $m\in [M]$; 
            \item $\widehat{\Lc}_m(f_{\omega};\Sc_m)$ is $\beta_{\omega}$-smooth w.r.t. $\omega$ (see Definition~\ref{def:smooth}) for all $m\in [M]$;
            \item the stochastic gradients are unbiased and have a $\sigma^2_b$-bounded variance (see Definition~\ref{def:bounded_variance}),
        \end{itemize}
        with SCAFFOLD as the adopted FL protocol, the output $\widehat{\omega}$ satisfies that
        \begin{align*}
            \Eb_{\xi}[\widehat{\Lc}(f_{\widehat{\omega}_{\Sc}}; \Sc) - \widehat{\Lc}(f_{\omega^*_{\Sc}}; \Sc) \mid \Sc] \leq   \tilde{O}\left(\frac{\sigma_b^2}{\mu_{\omega}\rho \kappa M}  + \mu_{\omega} \tilde{D}^2 \exp\left(- \min\left\{\frac{\rho}{30}, \frac{\mu_{\omega}\rho}{162\beta_{\omega}}\right\}\right) \right)
        \end{align*}
        when $\rho \geq \max\{\frac{162\beta_{\omega}}{\mu_{\omega}}, 30\}$, where $\rho$ denotes the round of communications (i.e., number of global aggregations), $\kappa$ is the number of local updates (i.e., SGD) between each communication, $\tilde{D}^2$ is a distant measure w.r.t. the initialization defined in \citet{karimireddy2020scaffold}. Note that the last term which decays exponentially w.r.t. $\rho$ is omitted in Corollary~\ref{col:convex_raw_full} and the following derivations for simplicity.
    \end{lemma}

    \subsection{LSGD-PFL}\label{subapp:LSGD_PFL}
    The LSGD-PFL protocol is summarized in \citet{hanzely2021personalized}, which is a general design for personalized federated learning problems. It largely follows FedAvg \citep{mcmahan2017communication}, while only the globally shared parameters are communicated and aggregated. The following lemma is provided in \citet{hanzely2021personalized} to characterize the convergence of LSGD-PFL.
    \begin{lemma}[Theorem 1 in \cite{hanzely2021personalized}]\label{lem:LSGD-PFL_full}
        For any dataset $\Sc$, if
        \begin{itemize}
            \item $\widehat{\Lc}_m(f_{\omega_m};\Sc_m)$ is $\mu_{\omega}$-strongly convex w.r.t. $\omega_m$ (see Definition~\ref{def:strongly_convex}) for all $m\in [M]$;
            \item $\widehat{\Lc}_m(f_{\omega^{\alpha}, \omega_m^{\beta}};\Sc_m)$ is $\beta_{\omega^\alpha}$-smooth w.r.t. $\omega^\alpha$ and $M\beta_{\omega^\beta}$-smooth w.r.t. $\omega^\beta_m$ (see Definition~\ref{def:smooth}) for all $m\in [M]$;
            \item the stochastic gradients w.r.t. $\omega^{\alpha}$ is unbiased and have a $\sigma_b^2$-bounded variance (see Definition~\ref{def:bounded_variance});
            \item the stochastic gradients w.r.t. $\{\omega^{\beta}_m: m\in [M]\}$ is unbiased and have a $\sigma_b^2$-bounded variance (see Definition~\ref{def:bounded_variance});
            \item the gradients w.r.t. $\omega$ have $G_b$ bounded dissimilarity (see Definition~\ref{def:bounded_dissimilarity}),
        \end{itemize}
        with LSGD-PFL as the adopted FL protocol, the output $\widehat{\omega}$  has
        $\varepsilon_{\text{opt}}(\Fc_{[M]}; n_{[M]}) \leq \varepsilon'$ after
        \begin{align*}
            \tilde{O}\left(\frac{\max\{\beta_{\omega^\beta }\kappa^{-1}, \beta_{\omega^\alpha}\}}{\mu_{\omega}} + \frac{\sigma^2_b}{\mu_{\omega} \kappa M \varepsilon'} + \frac{1}{\mu_{\omega}}\sqrt{\frac{\beta_{\omega^\alpha}(G^2+ \sigma^2)}{\varepsilon'}}\right)
        \end{align*}rounds of communications, where $\kappa$ is the number of local updates.
    \end{lemma}
    
    \section{Experiment Details}\label{app:exp}
    This section first provides a comprehensive description of the experimental settings and procedures. {\bf The codes and detailed instructions have been uploaded in the supplementary materials so as to execute the experiments and reproduce the results.}

    \textbf{Experimental details.} In the experiments, the system is designed as a synchronous one, i.e., $t_m(t) = t, \forall m\in [M]$, and for both tasks, two-layer multi-layer perceptrons (MLPs) with a hidden layer having a constant $256$ width are used to approximate the reward functions. 

    For practical conveniences, instead of selecting a theoretically sound but sophisticated choice of $\gamma$ for FedIGW as in Theorem~\ref{thm:global_regret}, we set it as a constant hyper-parameter and perform some preliminary manual selections with the final adopted values reported in Table~\ref{tab:hyper}. We believe this approach is more practically appealing as it does not need to scale $\gamma$ consistently; a similar choice of using constant $\gamma$'s is also adopted in \cite{agarwal2023empirical}. {Also, the temperature parameter $\zeta$ used in softmax can be found in Table~\ref{tab:hyper}.}

        \begin{table}[htb]
        \centering
        \caption{Hyperparameter choices for FedIGW in Bibtex and Delicious}
        \begin{tabular}{c|c|c|c|c|c}
        \hline
            Task & Learning Rate & Batch Size & Communications & Parameter  $\gamma$ & {Parameter $\zeta$}\\
            Bibtex &  0.1 & 64 & 100 & 7000 & {0.02}\\
            Delicious & 0.2 & 64 & 100 & 7000 & {0.02}\\
        \hline
        \end{tabular}
        \label{tab:hyper}
    \end{table}
    
    Multiple standard FL protocols including FedAvg \citep{mcmahan2017communication}, SCAFFOLD \citep{karimireddy2020scaffold} and FedProx \citep{li2020federated} are adopted as the FL component in FedIGW. During each FL process, the local batch size, the number of communications, and the local learning rate are specified in Table~\ref{tab:hyper}. Moreover, the epoch length is designed to be growing exponentially as in Corollaries~\ref{col:finite}, \ref{col:linear} and \ref{col:linear_per}, i.e., $\tau^l = 2^l$, while culminating at an upper limit of $4096$ to maintain timely updates. {The same FedAvg setup is also used in experiments with greedy and softmax to ensure fair comparisons.}

    \textbf{Additional comparisons with single-agent baselines.} In Fig.~\ref{fig:exp_main}, comparisons between FedIGW and the state-of-the-art FN-UCB are provided, demonstrating the superiority of FedIGW. Here we further report Fig.~\ref{fig:exp_single}, containing comparisons between FedIGW and two single-agent baselines:
    \begin{itemize}
        \item \textit{AGR.} The adaptive greedy (AGR) algorithm \citep{chakrabarti2008mortal} is selected as one of the single-agent baselines due to its strong empirical performance on Bibtex and Delicious reported in \cite{cortes2018adapting}. The algorithmic details can be found in \cite{cortes2018adapting}, and we also leveraged the code provided in \cite{cortes2018adapting} to build this baseline.

        \item \textit{FALCON.} The other single-agent baseline, FALCON, is proposed in \cite{simchi2022bypassing}, which is essentially the single-agent version of FedIGW. We still adopt the same algorithmic configurations as FedIGW (i.e., epoch length, parameter $\gamma$, local batch size, and local learning rate) except that the MLP is optimized locally instead of in a federation, i.e., there are no communications.
    \end{itemize}

    It can be observed that FedIGW (with $M=10$ participating agents and the basic FedAvg) can outperform the two single-agent baselines on both tasks, demonstrating the benefit of learning in a federation.

    \begin{figure}[htb]
	\setlength{\abovecaptionskip}{-2pt} 
        \centering
        \subfigure{
        \includegraphics[width=0.48\linewidth]{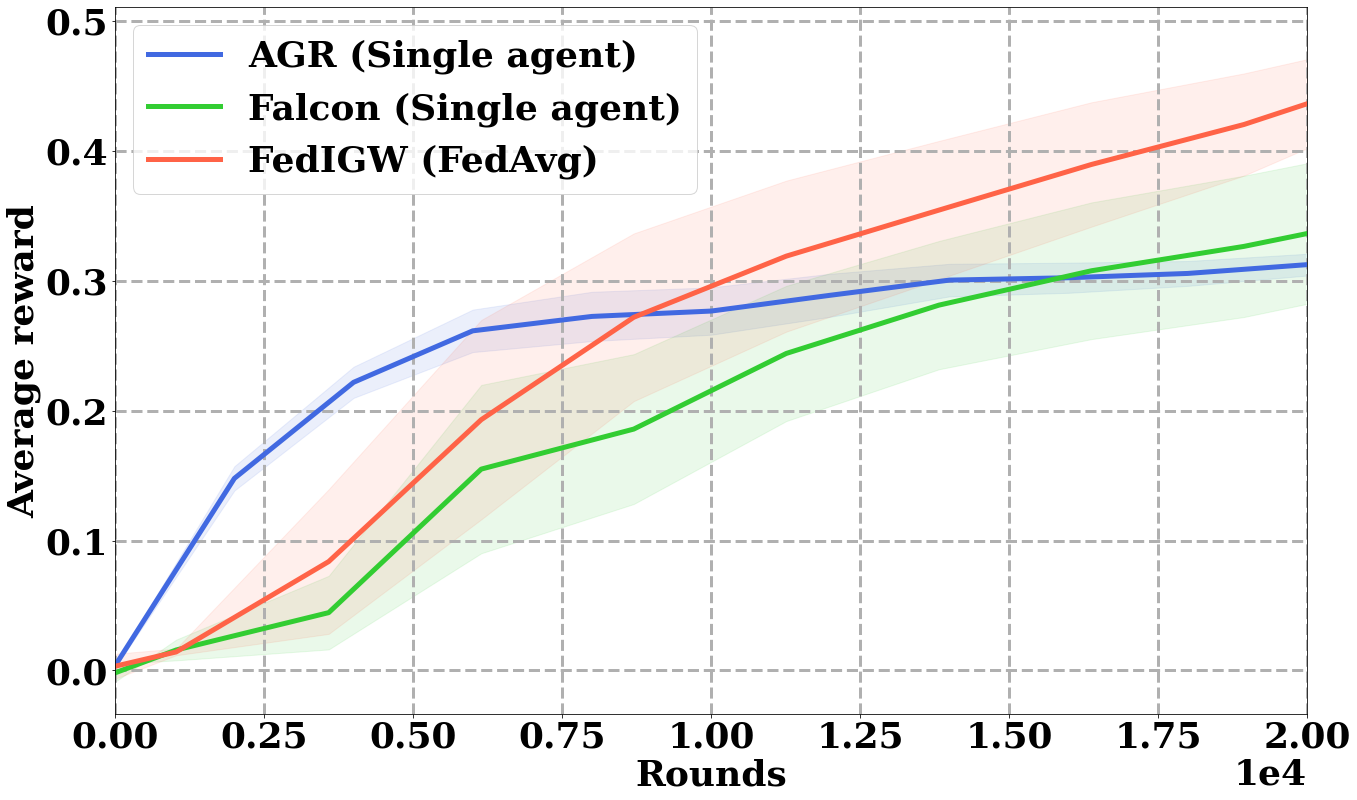}}
        \subfigure{
        \includegraphics[width=0.48\linewidth]{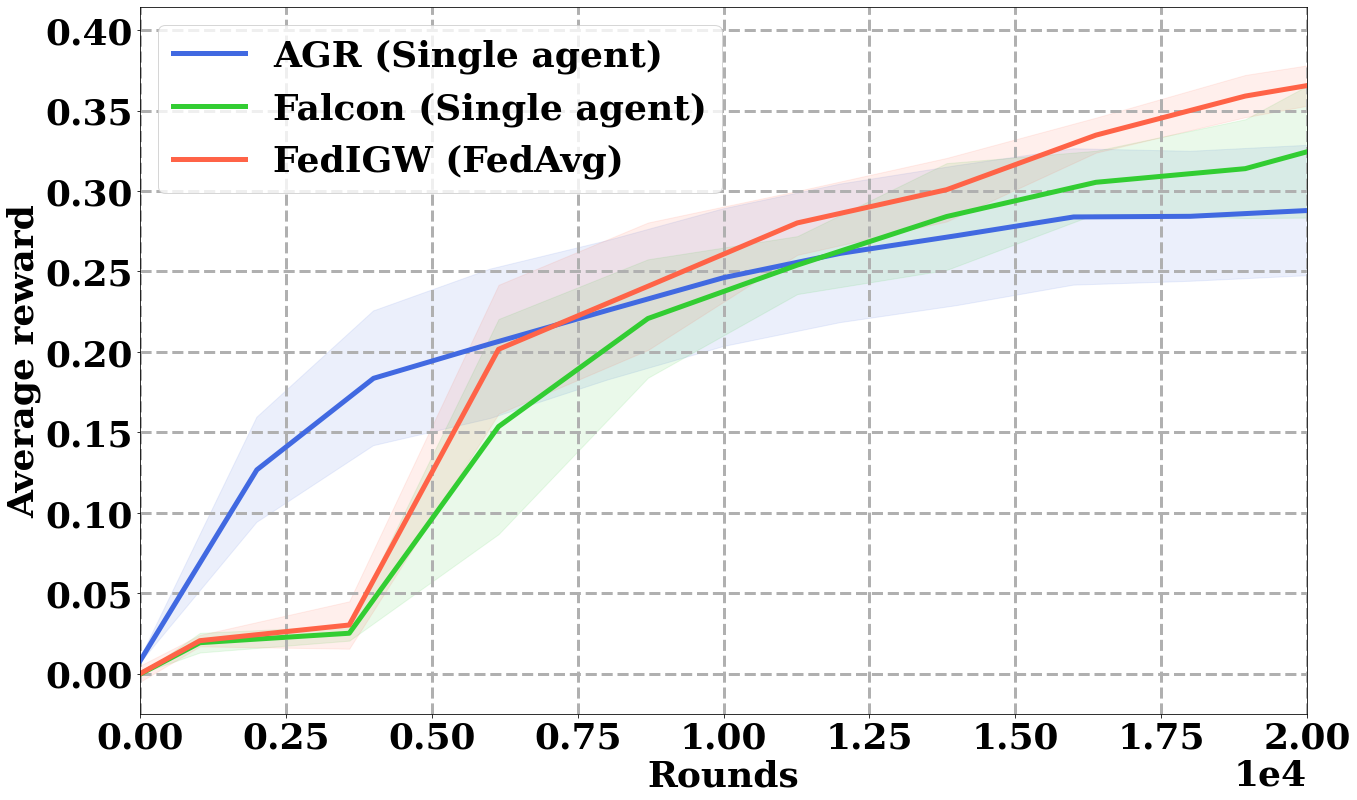}}
        \caption{\centering The averaged reward collected by each agent via FedIGW (with FedAvg and $M = 10$ participating agents) and two single-agent baselines on Bibtex (left) and Delicious (right) datasets.}
        \label{fig:exp_single}%
    \end{figure}

\end{document}